\documentclass{colt2016} 

\title[Delay and Cooperation in Nonstochastic Bandits]{Delay and Cooperation in Nonstochastic Bandits}

\usepackage{times}

\usepackage{hyperref}
\usepackage{url}
\usepackage{color}

\usepackage{amsfonts}
\usepackage{amsmath}
\usepackage{graphicx}

\usepackage{algorithm}
\usepackage{algorithmic}
\usepackage{ifthen}
\usepackage[algo2e,ruled]{algorithm2e}

\usepackage{color,soul} 
\usepackage[parfill]{parskip} 

\newcommand{\qed}{\hfill $\Box$}
\newenvironment{proofof}[1]{\par\noindent{\bf Proof of #1.}}{\qed\par\smallskip\noindent}

\newcommand{\field}[1]{\mathbb{#1}}

\newcommand{\E}{\field{E}}
\renewcommand{\Pr}{\field{P}}
\newcommand{\Ind}[1]{\field{I}{\{#1\}}}

\newcommand{\dt}{\displaystyle}

\newcommand{\scP}{\mathcal{P}}

\newcommand{\bp}{\boldsymbol{p}}

\newcommand{\loss}{\ell}
\newcommand{\hloss}{\widehat{\loss}}
\newcommand{\bloss}{\boldsymbol{\loss}}

\newlength{\minipagewidth}
\newcommand{\bookbox}[1]{
\par\medskip\noindent
\framebox[\textwidth]{
\begin{minipage}{\minipagewidth}
{#1}
\end{minipage} } \par\medskip }

\newcommand{\tprob}{\widetilde{p}}
\newcommand{\TProb}{\widetilde{P}}
\newcommand{\avdv}{\widehat{d}}

\newcommand{\dist}{\mathrm{dist}}

\coltauthor
{
 \Name{\!Nicol\`o Cesa-Bianchi} \Email{nicolo.cesa-bianchi@unimi.it}\\
 \addr Universit\`a degli Studi di Milano, Italy
 \AND
 \Name{Claudio Gentile} \Email{claudio.gentile@uninsubria.it}\\
 \addr University of Insubria, Italy
 \AND
 \Name{Yishay Mansour} \Email{mansour@tau.ac.il}\\
 \addr Microsoft Research and Tel-Aviv University, Israel
 \AND
 \Name{Alberto Minora} \Email{aminora@uninsubria.it}\\
 \addr University of Insubria, Italy
}

\begin{document}

\maketitle
\begin{abstract}
We study networks of communicating learning agents that cooperate to solve a common nonstochastic bandit problem. Agents use an underlying communication network to get messages about actions selected by other agents, and drop messages that took more than $d$ hops to arrive, where $d$ is a delay parameter. We introduce \textsc{Exp3-Coop}, a cooperative version of the {\sc Exp3} algorithm and prove that with $K$ actions and $N$ agents the average per-agent regret after $T$ rounds is at most of order $\sqrt{\bigl(d+1 + \tfrac{K}{N}\alpha_{\le d}\bigr)(T\ln K)}$, where $\alpha_{\le d}$ is the independence number of the $d$-th power of the communication graph $G$.
%
We then show that for any connected graph, for $d=\sqrt{K}$  the regret bound is $K^{1/4}\sqrt{T}$, strictly better than the minimax regret $\sqrt{KT}$ for noncooperating agents.
More informed choices of $d$ lead to bounds which are arbitrarily close to the full information minimax regret $\sqrt{T\ln K}$ when $G$ is dense. When $G$ has sparse components, we show that a variant of \textsc{Exp3-Coop}, allowing agents to choose their parameters according to their centrality in $G$, strictly improves the regret. Finally, as a by-product of our analysis, we provide the first characterization of the minimax regret for bandit learning with delay.
\end{abstract}


\section{Introduction}
\label{s:intro}
%
Delayed feedback naturally arises in many sequential decision problems. For instance, a recommender system typically learns the utility of a recommendation by detecting the occurrence of certain events (e.g., a user conversion), which may happen with a variable delay after the recommendation was issued. Other examples are the communication delays experienced by interacting learning agents. Concretely, consider a network of geographically distributed ad servers using real-time bidding to sell their inventory. Each server sequentially learns how to set the auction parameters (e.g., reserve price) in order to maximize the network's overall revenue, and shares feedback information with other servers in order to speed up learning. However, the rate at which information is exchanged through the communication network is slower than the typical rate at which ads are served. This causes each learner to acquire feedback information from other servers with a delay that depends on the network's structure.



Motivated by the ad network example, we consider networks of learning agents that cooperate to solve the same nonstochastic bandit problem, and study the impact of delay on the global performance of these agents.
We introduce the {\sc Exp3-Coop} algorithm, a distributed and cooperative version of the {\sc Exp3} algorithm of \cite{auer2002nonstochastic}.
{\sc Exp3-Coop} works within a distributed and synced model where each agent runs an instance of the same bandit algorithm ({\sc Exp3}). All bandit instances are initialized in the same way irrespective to the agent's location in the network (that is, agents have no preliminary knowledge of the network), and we assume the information about an agent's actions is propagated through the network with a unit delay for each crossed edge. In each round $t$, each agent selects an action and incurs the corresponding loss (which is the same for all agents that pick that action in round $t$). Besides observing the loss of the selected action, each agent obtains the information previously broadcast by other agents with a delay equal to the shortest-path distance between the agents. Namely, at time $t$ an agent learns what the agents at shortest-path distance $s$ did at time $t-s$ for each $s = 1, \ldots, d$, where $d$ is a delay parameter. In this scenario, we aim at controlling the growth of the regret averaged over all agents (the so-called average welfare regret).

In the noncooperative case, when agents ignore the information received from other agents, the average welfare regret grows like $\sqrt{KT}$ (the minimax rate for standard bandit setting), where $K$ is the number of actions and $T$ is the time horizon. We show that, using cooperation, $N$ agents with communication graph $G$ can achieve
an average welfare regret of order $\sqrt{\bigl(d+1 + \tfrac{K}{N}\alpha_{\leq d}\bigr)(T\ln K)}$. Here $\alpha_{\leq d}$ denotes the independence number of the $d$-th power of $G$ (i.e., the graph $G$ augmented with all edges between any two pair of nodes at shortest-path distance less than or equal to $d$). When $d = \sqrt{K}$ this bound is at most $K^{1/4}\sqrt{T\ln K} + \sqrt{K}(\ln T)$
for any connected graph
---see Remark~\ref{rem:choiceofd} in Section~\ref{s:single-d}--- which is asymptotically better than $\sqrt{KT}$.

Networks of nonstochastic bandits were also investigated by~\cite{awerbuch2008competitive} in a setting where the distribution over actions is shared among the agents without delay. \cite{awerbuch2008competitive} prove a bound on the average welfare regret of order $\sqrt{\bigl(1 + \tfrac{K}{N}\bigr)T}$ ignoring polylog factors.\footnote{
The rate proven in~\citep[Theorem~2.1]{awerbuch2008competitive} has a worse dependence on $T$, but we believe this is due to the fact that their setting allows for dishonest agents and agent-specific loss vectors.
}
We recover the same bound as a special case of our bound when $G$ is a clique and $d=1$. In the clique case our bound is also similar to the bound $\sqrt{\tfrac{K}{N}(T\ln K)}$ achieved by~\cite{seldin2014prediction} in a single-agent bandit setting where, at each time step, the agent can choose a subset of $N \le K$ actions and observe their loss.
In the case when $N=1$ (single agent), our analysis can be applied to the nonstochastic bandit problem where the player observes the loss of each played action with a delay of $d$ steps.
In this case we improve on the previous result of $\sqrt{(d+1)KT}$ by~\cite{neu2010online,neu2014online},
and give the first characterization (up to logarithmic factors) of the minimax regret, which is of order $\sqrt{(d + K)\,T}$.

In principle, the problem of delays in online learning could be tackled by simple reductions. Yet, these reductions give rise to suboptimal results. In the single agent setting, where the delay is constant and equal to $d$, one can use the technique of~\cite{weinberger2002delayed} and run $d+1$ instances of an online algorithm for the nondelayed case, where each instance is used every $d+1$ steps. This delivers a suboptimal regret bound of $\sqrt{(d+1)KT}$. In the case of multiple delays, like in our multi-agent setting, one can repeat the same action for $d+1$ steps while accumulating information from the other agents, and then perform an update on scaled-up losses. The resulting (suboptimal) bound on the average welfare regret would be of the form $\sqrt{(d+1)\bigl(1 + \tfrac{K}{N}\alpha_{\leq d}\bigr)(T\ln K)}$.

Rather than using reductions, the analysis of {\sc Exp3-Coop} rests on quantifying the performance of
suitable importance weighted estimates. In fact, in the single-agent setting with delay parameter $d$, using {\sc Exp3-Coop} reduces to
running the standard {\sc Exp3} algorithm performing an update as soon a new loss becomes available. This implies that at any round $t > d$, {\sc Exp3} selects an action without knowing the losses incurred during the last $d$ rounds. The resulting regret is bounded by relating the standard analysis of {\sc Exp3} to a detailed quantification of the extent to which the distribution maintained by {\sc Exp3} can drift in $d$ steps.

In the multi-agent case, the importance weighted estimate of \textsc{Exp3-Coop} is designed in such a way that at each time $t > d$ the instance of the algorithm run by an agent $v$ updates all actions that were played at time $t-d$ by agent $v$ or by other agents not further away than $d$ from $v$. Compared to the single agent case, here each agent can exploit the information circulated by the other agents. However, in order to 
compute the importance weighted estimates used locally by each agent, the probabilities maintained by the agents must be propagated together with the observed losses. Here, further concerns may show up, like the amount of communication, and the location of each agent within the network. In particular, when $G$ has sparse components, we show that a variant of \textsc{Exp3-Coop}, allowing agents to choose their parameters according to their centrality within $G$, strictly improves on the regret of \textsc{Exp3-Coop}.

\section{Additional Related Work}
\label{s:related}
Many important ideas in delayed online learning, including the observation that the effect of delays can be limited by controlling the amount of change in the agent strategy, were introduced by~\citet{mesterharm2005line} ---see also \cite[Chapter~8]{Mester2007}. A more recent investigation on delayed online learning is due to~\citet{neu2010online,neu2014online}, who analyzed exponential weights with delayed feedbacks. Furher progress is made by~\citet{joulani2013online}, who also study delays in the general partial monitoring setting. Additional works~\citep{DBLP:conf/aaai/JoulaniGS16,NIPS2015_5833} prove regret bounds for the full-information case of the form $\sqrt{(D+T)\ln K}$, where $D$ is the total delay experienced over the $T$ rounds. In the stochastic case, bandit learning with delayed feedback was considered by \citet{DBLP:conf/uai/DudikHKKLRZ11,joulani2013online}.

To the best of our knowledge, the first paper about nonstochastic cooperative bandit networks is \citep{awerbuch2008competitive}. More papers analyze the stochastic setting, and the closest one to our work is perhaps~\citep{szorenyi2013gossip}. In that paper, delayed loss estimates in a network of cooperating stochastic bandits are analyzed using a dynamic P2P random networks as communication model. A more recent paper is~\citep{landgren2015distributed}, where the communication network is a fixed graph and a cooperative version of the UCB algorithm is introduced which uses a distributed consensus algorithm to estimate the mean rewards of the arms. The main result is an individual (per-agent) regret bound that depends on the network structure without taking delays into account.
Another interesting paper about cooperating bandits in a stochastic setting is~\citep{kar2011bandit}. Similar to our model, agents sit on the nodes of a communication network. However, only one designated agent observes the rewards of actions he selects, whereas the others remain in the dark. This designated agent broadcasts his sampled actions through the networks to the other agents, who must learn their policies relying only on this indirect feedback. The paper shows that in any connected network this information is sufficient to achieve asymptotically optimal regret.
Cooperative bandits with asymmetric feedback are also studied by \cite{barrett2011ad}. In their model, an agent must teach the reward distribution to another agent while keeping the discounted regret under control. \cite{TekinS15} investigate a stochastic contextual bandit model where each agent can either privately select an action or have another agent select an action on his behalf. In a related paper, \cite{TekinZS14} look at a stochastic bandit model with combinatorial actions in a distributed recommender system setting, and study incentives among agents who can now recommend items taken from other agents' inventories.
Another line of relevant work involves problems of decentralized bandit coordination. For example, \cite{stranders2012dcops} consider a bandit coordination problem where the the reward function is global and can be represented as a factor graph in which each agent controls a subset of the variables.
A parallel thread of research concerns networks of bandits that compete for shared resources. A paradigmatic application domain is that of cognitive radio networks, in which a number of channels are shared among many users and any two or more users interfere whenever they simultaneously try to use the same channel. The resulting bandit problem is one of coordination in a competitive environment, because every time two or more agents select the same action at the same time step they both get a zero reward due to the interference ---see~\citep{RosenskiSS15} for recent work on stochastic competitive bandits and~\citep{kleinberg2009multiplicative} for a study of more general congestion games in a game-theoretic setting.
Finally, there exists an extensive literature on the adaptation of gradient descent and related algorithms to distributed computing settings, where asynchronous processors naturally introduce delays ---see, e.g., \citep{NIPS2009_3888,NIPS2011_4247,li2013distributed,NIPS2014_5242,NIPS2015_5833,liu2015asynchronous,duchi2015asynchronous}. However, none of these works considers bandit settings, which are an essential ingredient for our analysis.

\section{Preliminaries}\label{s:prel}
We now establish our notation, along with basic assumptions and preliminary facts related to our algorithms. Notation and setting here both refer to the single agent case. The cooperative setting with multiple agents (and notation thereof) will be introduced in Section~\ref{s:multi}. Proofs of all the results stated here can be found in~\citep{cesa2016delay}.

Let $A = \{1,\dots,K\}$ be the action set. A learning agent runs an exponentially-weighted algorithm with weights $w_t(i)$, and learning rate $\eta > 0$. Initially, $w_1(i) = 1$ for all $i \in A$. At each time step $t=1,2,\dots$, the agent draws action $I_t$ with probability $\Pr(I_t = i) = p_t(i) = w_t(i)/W_t$, where $W_t = \sum_{j \in A} w_t(j)$. After observing the loss $\loss_t(I_t) \in [0,1]$ associated with the chosen action $I_t$, and possibly some additional information, the agent computes, for each $i \in A$, nonnegative loss estimates $\hloss_t(i)$, and performs the exponential update
\begin{equation}
\label{eq:exp-upd}
    w_{t+1}(i) = p_t(i)\,\exp\bigl(-\eta\,\hloss_t(i)\bigr)
\end{equation}
to these weights. The following two lemmas
are general results that control the evolution of the probability distributions in the exponentially-weighted algorithm. As we said in the introduction, bounding the extent to which the distribution used by our algorithms can drift in $d$ steps is key to controlling regret in a delayed setting. The first result bounds the {\em additive} change in the probability of any action, and it holds no matter how $\hloss_t(i)$ is defined.
\begin{lemma}
\label{l:sandwich}
Under the update rule~(\ref{eq:exp-upd}), for all $t \ge 1$ and for all $i \in A$,
\[
    -\eta\,p_t(i)\hloss_t(i) \le p_{t+1}(i) - p_t(i) \le \eta\,p_{t+1}(i)\sum_{j \in A} p_t(j)\hloss_t(j)
\]
holds deterministically with respect to the agent's randomization.
\end{lemma}
The second result delivers a {\em multiplicative} bound on the change in the probability of any action when the loss estimates
$\hloss_t(i)$ are of the following form:
\begin{equation}\label{e:lossestimate}
    \hloss_t(i) = \left\{ \begin{array}{cl}
        \displaystyle{\frac{\loss_{t-d}(i)}{q_{t-d}(i)}}\, B_{t-d}(i) & \text{if $t > d$,}
        \\
        0 & \text{otherwise~,}
    \end{array} \right.
\end{equation}
where $d \geq 0$ is a delay parameter, $B_{t-d}(i) \in \{0,1\}$, for $i \in A$, are indicator functions, and $q_{t-d}(i) \ge p_{t-d}(i)$ for all $i$ and $t > d$. In all later sections, $B_{t-d}(i)$ will be instantiated to the indicator function of the event that action $i$ has been played at time $t-d$ by some agent, and $q_{t-d}(i)$ will be the (conditional) probability of this event.
\begin{lemma}
\label{l:mult}
Let $\hloss_t(i)$ be of the form (\ref{e:lossestimate}) for each $t \ge 1$ and $i \in A$.
%
If $\eta \le \frac{1}{Ke(d+1)}$ in the update rule~(\ref{eq:exp-upd}), then
\[
    p_{t+1}(i) \le \left(1 + \frac{1}{d}\right)p_t(i)
\]
holds for all $t \ge 1$ and $i \in A$, deterministically with respect to the agent's randomization.
\end{lemma}
As we said in Section~\ref{s:intro}, the idea of controlling the drift of the probabilities in order to bound the effects of delayed feedback is not new. In particular, variants of Lemma~\ref{l:sandwich} were already derived in the work of~\cite{neu2010online,neu2014online}. However, Lemma~\ref{l:mult} appears to be new, and this is the key result to achieving our improvements.

\section{The Cooperative Setting on a Communication Network}\label{s:multi}
In our multi-agent bandit setting, there are $N$ agents sitting on the vertices of a connected and undirected communication graph $G = (V,E)$, with $V = \{1, \ldots, N\}$. The agents cooperate to solve the same instance of a nonstochastic bandit problem while limiting the communication among them. Let $N_s(v)$ be the set of nodes $v' \in V$ whose shortest-path distance $\dist_G(v,v')$ from $v$ in $G$ is exactly $s$.
At each time step $t=1,2,\dots$, each agent $v \in V$ draws an action $I_t(v)$ from the common action set $A$. Note that each action $i \in A$ delivers the same loss $\loss_t(i) \in [0,1]$ to all agents $v$ such that $I_t(v) = i$. At the end of round $t$, each agent $v$ observes his own loss $\loss_t\bigl(I_t(v)\bigr)$, and sends to his neighbors in $G$ the message
\[
    m_t(v) = \Bigl\langle t,v,I_t(v),\loss_t\bigl(I_t(v)\bigr),\bp_t(v) \Bigr\rangle
\]
where $\bp_t(v) = \bigl(p_t(1,v),\dots,p_t(K,v)\bigr)$ is the distribution of $I_t(v)$. Moreover, $v$ also receives from his neighbors a variable number of messages $m_{t-s}(v')$. Each message $m_{t-s}(v')$ that $v$ receives from a neighbor is used to update $\bp_t(v)$ and then forwarded to the other neighbors only if $s < d$,
otherwise it is dropped.\footnote
{
Dropping messages older than $d$ rounds is clearly immaterial with respect to proving bandit regret bounds. We added this feature just to prove a point about the message complexity of the protocol. See Remark~\ref{r:exo} in Section~\ref{s:many-d} for further discussion.
}
Here $d$ is the maximum delay, a parameter of the communication protocol. Therefore, at the end of round $t$, each agent $v$ receives one message $m_{t-s}(v')$ for each agent $v'$ such that $\dist_G(v,v') = s$, where $s\in\{1,\dots,d\}$.
Graph $G$ can thus be seen as a synchronous multi-hop communication network where messages are broadcast, each hop causing a delay of one time step.
Our learning protocol is summarized in Figure~\ref{f:protocol}, while Figure~\ref{f:example} contains a pictorial example.

Our model is
similar to the {\sc local} communication model in distributed computing \citep{Linial92,Suomela13}, where the output of a node depends only on the inputs of other nodes in a constant-size neighborhood of it,
and the goal is to derive algorithms whose running time is independent of the network size.
(The main difference is that the task here has no completion time, however, also in our model influence on a node is only through a constant-size neighborhood of it.)

\begin{figure}[t]
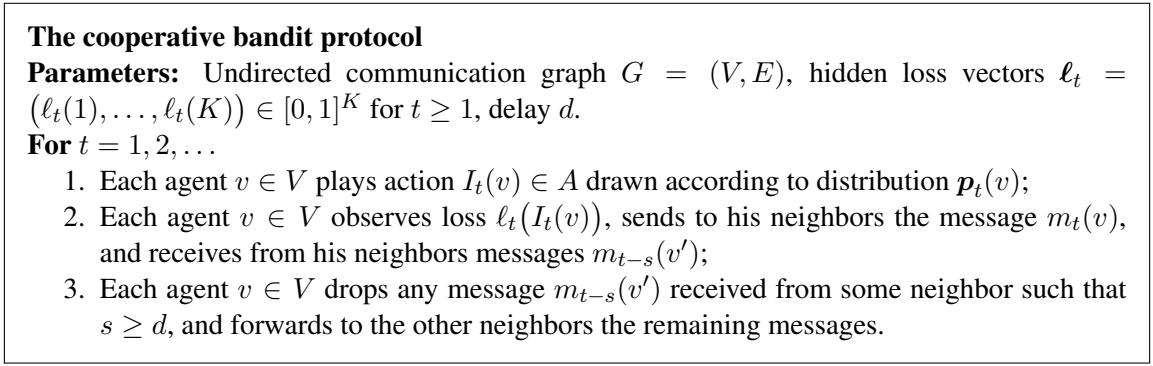

\bookbox{
\textbf{The cooperative bandit protocol}\\
\textbf{Parameters:} Undirected communication graph $G = (V,E)$, hidden loss vectors $\bloss_t = \big(\ell_t(1), \ldots, \ell_t(K)\big) \in [0,1]^K$ for $t \ge 1$, delay $d$.\\
\textbf{For $t=1,2,\dots$}
\begin{enumerate}
\item Each agent $v \in V$ plays action $I_t(v) \in A$ drawn according to distribution
$
    \bp_t(v)
$;
\item\label{st:2}
Each agent $v \in V$ observes loss $\loss_t\bigl(I_t(v)\bigr)$, sends to his neighbors the message $m_t(v)$, and receives from his neighbors messages $m_{t-s}(v')$;
\item\label{st:3}
Each agent $v \in V$ drops any message $m_{t-s}(v')$ received from some neighbor such that $s \ge d$, and forwards to the other neighbors the remaining messages.
\end{enumerate}
}
\caption{
The cooperative bandit protocol where all agents share the same delay parameter $d$.
}
\label{f:protocol}
\end{figure}


One aspect deserving attention is that, apart from the common delay parameter $d$, the agents need not share further information. In particular, the agents need not know neither the topology of the graph $G$ nor the total number of agents $N$. In Section~\ref{s:many-d}, we show that our distributed algorithm can also be analyzed when each agent $v$ uses a personalized delay $d(v)$, thus doing away with the need of a common delay parameter, and guaranteeing a generally better performance.

Further graph notation is needed at this point.
Given $G$ as above, let us denote by $G_{\le d}$ the graph $(V,E_{\le d})$ where $(u,v) \in E_{\le d}$ if and only if the shortest-path distance between agents $u$ and $v$ in $G$ is {\em at most} $d$ (hence $G_{\le 1} = G$).
Graph $G_{\le d}$ is sometimes called the $d$-th power of $G$. We also use $G_0$ to denote the graph $(V,\emptyset)$.
Recall that an independent set of $G$ is any subset
$T \subseteq V$ such that no two $i,j \in T$ are connected by an edge in $E$.
The largest size of an independent set is the {\em independence number} of $G$, denoted by $\alpha(G)$.
Let ${d_G}$ be the {\em diameter} of $G$
(maximal length over all possible shortest paths between all pairs of nodes); then $G_{\le d_G}$ is a clique, and one can easily see that $N = \alpha(G_0) > \alpha(G) \ge \alpha(G_{\le 2}) \ge \cdots \ge \alpha(G_{\le {d_G}}) = 1$. We show in Section~\ref{s:single-d} that the collective performance of our algorithms depends on $\alpha(G_{\le d})$. If the graph $G$ under consideration is directed (see Section~\ref{s:many-d}), then $\alpha(G)$ is the independence number of the undirected graph obtained from $G$ by disregarding edge orientation.

%
The adversary generating losses is oblivious: loss vectors $\bloss_t = \big(\ell_t(1), \ldots, \ell_t(K)\big) \in [0,1]^K$ do not depend on the agents' internal randomization.
The agents' goal is to control the {\em average welfare} regret $R_T^{\mathrm{coop}}$, defined as
\[
    R_T^{\mathrm{coop}} = \left( \frac{1}{N}\sum_{v \in V}\E\left[ \sum_{t=1}^T \loss_t\bigl(I_t(v)\bigr)\right] -  \min_{i \in A} \sum_{t=1}^T \loss_t(i) \right)~,
\]
the expectation being with respect to the internal randomization of each agent's algorithm.
In the sequel, we write $\E_{t}[\cdot]$ to denote the expectation w.r.t.\ the product distribution $\prod_{v \in V} \bp_t(v)$, conditioned on $I_1(v),\dots,I_{t-1}(v)$, $v \in V$.

\begin{figure}[t!]
\begin{picture}(-40,270)(-40,270)
\scalebox{0.65}{\includegraphics{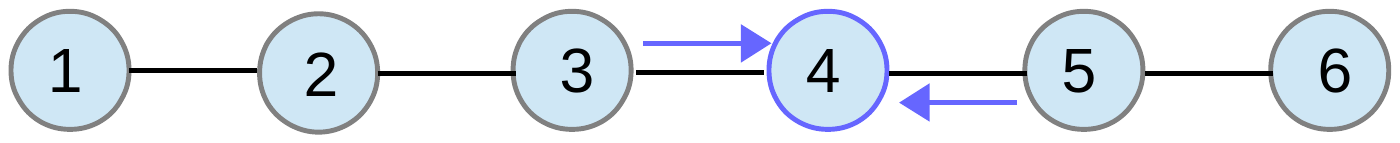}}
\end{picture}
\vspace{-3.45in}
\caption{\label{f:example}
In this example, $G$ is a line graph with $N = 6$ agents, and delay $d = 2$. At the end of time step $t$, agent $4$
sends to his neighbors $3$ and $5$ message $m_t(4)$,
receives from agent $3$ messages $m_{t-1}(3)$, and $m_{t-2}(2)$, and from agent $5$ messages $m_{t-1}(5)$ and $m_{t-2}(6)$. Finally, $4$ forwards to $5$ message $m_{t-1}(3)$ and forwards to $3$ message $m_{t-1}(5)$.
Any message older than $t-1$ received by $4$ at the end of round $t$ will not be forwarded to his neighbors.
}
\end{figure}

\subsection{The Exp3-Coop algorithm}
\label{s:single-d}
Our first algorithm, called {\sc Exp3-Coop} (Cooperative Exp3) is described in Figure \ref{f:exp3-coop}. The algorithm works in the learning protocol of Figure~\ref{f:protocol}. Each agent $v \in V$ runs the exponentially-weighted algorithm~(\ref{eq:exp-upd}), combined with a ``delayed'' importance-weighted loss estimate $\hloss_t(i,v)$ that incorporates the delayed information sent by the other agents. Specifically, denote by
$
    N_{\le d}(v) = \bigcup_{s \le d} N_s(v)
$
the set of nodes in $G$ whose shortest-path distance from $v$ is at most $d$, and note that, for all $v$, $\{v\} = N_{\leq 0}(v) \subseteq N_{\leq 1}(v) \subseteq N_{\leq 2}(v) \subseteq  \cdots $\,. If any of the agents in $N_{\leq d}(v)$ has played at time $t-d$ action $i$ (that is, $B_{d,t-d}(i,v) = 1$ in Eq. in (\ref{eq:estimator})), then the corresponding loss $\ell_{t-d}(i)$ is incorporated by $v$ into $\hloss_t(i,v)$.
The denominator $q_{d,t-d}(i,v)$ is simply,
conditioned on the history,
the probability of $B_{d,t-d}(i,v) = 1$,
i.e., $q_{d,t-d}(i,v)=\E_t[B_{d,t-d}(i,v)]$.
Observe that $\{v\} \subseteq N_{\leq d}(v)$ for all $d \geq 0$ implies $q_{d,t-d}(i,v) \geq p_{t-d}(i,v)$, as required by (\ref{e:lossestimate}). It is also worth mentioning that, despite this is not strictly needed by our learning protocol, each agent $v$ actually exploits the loss information gathered from playing action $I_t(v)$ only $d$ time steps later. A relevant special case of this learning mode is when we only have a single bandit agent receiving delayed feedback (Section~\ref{s:delayed}).
\begin{figure}[t!]
\bookbox{
\textbf{The Exp3-Coop Algorithm}\\
\textbf{Parameters:} Undirected communication graph $G = (V,E)$; delay $d$; learning rate $\eta$.\\
\textbf{Init:} Each agent $v \in V$ sets weights $w_1(i,v) = 1$ for all $i \in A$.\\
\textbf{For $t=1,2,\dots$}
\begin{enumerate}
\item Each agent $v \in V$ plays action $I_t(v) \in A$ drawn according to distribution
$
    \bp_t(v) = (p_t(1,v),\ldots,p_t(K,v))~,
$
where
\[
p_t(i,v) = \frac{w_t(i,v)}{W_t(v)},\, i = 1, \ldots, K, \qquad\text{and}\qquad W_t(v) = \sum_{j \in A} w_t(j,v)~;
\]
\item Each agent $v \in V$ observes loss $\loss_t\bigl(I_t(v)\bigr)$ and exchanges messages with his neighbors (Steps~\ref{st:2} and~\ref{st:3} of the protocol in Figure~\ref{f:protocol});
\item Each agent $v \in V$ performs the update
$
w_{t+1}(i,v) = p_t(i,v)\,\exp\bigl(-\eta\,\hloss_t(i,v)\bigr)
$ for all $i \in A$,
where
\begin{equation}
\label{eq:estimator}
    \hloss_t(i,v) = \left\{ \begin{array}{cl}
        \displaystyle{\frac{\loss_{t-d}(i)}{q_{d,t-d}(i,v)}}\, B_{d,t-d}(i,v) & \text{if $t > d$,}
        \\
        0 & \text{otherwise,}
    \end{array} \right.
\end{equation}
and
${\dt
B_{d,t-d}(i,v) = \Ind{\exists v' \in N_{\le d}(v) \,:\, I_{t-d}(v') = i}
}$
with
\[
    q_{d,t-d}(i,v) = 1 - \prod_{v' \in N_{\le d}(v)}\bigl(1-p_{t-d}(i,v')\bigr)~.
\]
\end{enumerate}
}
\caption{
The Exp3-Coop algorithm where all agents share the same delay parameter $d$.
}
\label{f:exp3-coop}
\end{figure}

By their very definition, the loss estimates $\hloss_t(\cdot,\cdot)$ at time $t$ are determined by the realizations of $I_s(\cdot)$, for $s=1,\dots,t-d$. This implies that the numbers $p_t(\cdot,\cdot)$ defining $q_{d,t-d}(\cdot,\cdot)$, are determined by the realizations of $I_s(\cdot)$ for $s=1,\dots,t-d-1$ (because the probabilities $\bp_t(v)$ at time $t$ are determined by the loss estimates up to time $t-1$, see~(\ref{eq:exp-upd})). We have, for all $t > d$, $i \in A$, and $v \in V$,
\begin{equation}
\label{eq:avevar}
    \E_{t-d}\Bigl[\hloss_t(i,v)\Bigr] = \loss_{t-d}(i)~.
\end{equation}
%
\sloppypar
{
Further, because of what we just said about $p_t(\cdot,\cdot)$ and $q_{d,t-d}(\cdot,\cdot)$ being determined by $I_1(\cdot),\dots,I_{t-d-1}(\cdot)$, we also have
}
\begin{equation}
\label{eq:aveprob}
    \E_{t-d}\Bigl[p_t(i,v)\hloss_t(i,v)\Bigr] = p_t(i,v)\loss_{t-d}(i)~,
\quad
    \E_{t-d}\Bigl[p_t(i,v)\hloss_t(i,v)^2\Bigr] = p_t(i,v)\frac{\loss_{t-d}(i)^2}{q_{d,t-d}(i,v)}~.
\end{equation}
The following theorem quantifies the behavior of {\sc Exp3-Coop} in terms of a free parameter $\gamma$ in the learning rate, the tuning of which will be addressed in the subsequent Theorem~\ref{th:main}.
%
\begin{theorem}
\label{th:nontuned}
The regret of {\sc Exp3-Coop} run over a network $G = (V,E)$ of $N$ agents, each using delay $d$ and learning rate $\eta = \gamma\big/\bigl(Ke(d+1)\bigr)$, for $\gamma \in (0,1]$, satisfies
\[
    R_T^{\mathrm{coop}} \le 2d + \frac{Ke(d+1)\ln K}{\gamma} + \gamma\left(\frac{\alpha(G_{\le d})}{2(1-e^{-1})(d+1)N} + \frac{3}{Ke}\right)T~.
\]
\end{theorem}
With this bound handy, we might be tempted to optimize for $\gamma$. However, this is not a legal learning rate setting in a distributed scenario, for the optimized value of $\gamma$ would depend on the global quantities $N$ and $\alpha(G_{\leq d})$. Thus, instead of this global tuning, we let each agent set its own learning rate $\gamma$ through a ``doubling trick'' played locally. The doubling trick\footnote
{
There has been some recent work on adaptive learning rate tuning applied to nonstochastic bandit algorithms~\citep{k+14,neu15}. One might wonder whether the same techniques may apply here as well. Unfortunately, the specific form of our update~(\ref{eq:exp-upd}) makes this adaptation nontrivial, and this is why we resorted to a more traditional ``doubling trick".
}
works as follows. For each $v\in V$, we let
$\gamma_r(v) = Ke(d+1)\sqrt{(\ln K)/2^r}$
for each $r = r_0,r_0+1,\dots$, where $r_0 = \bigl\lceil\log_2\ln K + 2\log_2(Ke(d+1))\bigr\rceil$ is chosen in such a way that $\gamma_r(v) \le 1$ for all $r \ge r_0$. Let $T_r$ be the random set of consecutive time steps where the same $\gamma_r(v)$ was used. Whenever the local algorithm at $v$ is running with $\gamma_r(v)$ and detects $\sum_{s \in T_r} Q_s(v) > 2^r$, then we restart this algorithm with $\gamma(v) = \gamma_{r+1}(v)$. 

We have the following result.
\begin{theorem}
\label{th:main}
The regret of {\sc Exp3-Coop} run over a network $G = (V,E)$ of $N$ agents, each using delay $d$, and an individual learning rate $\eta(v) =  \gamma(v)/\bigl(Ke(d+1)\bigr)$, where $\gamma(v) \in (0,1]$ is adaptively selected by each agent through the above doubling trick, satisfies, when $T$ grows large,\footnote
{
The big-oh notation here hides additive terms that are independent of $T$ and do depend polynomially on the other parameters.
}
\begin{align*}
    R_T^{\mathrm{coop}}
&=
    \mathcal{O}\left(\sqrt{(\ln K)\left(d + 1+ \frac{K}{N}\,\alpha(G_{\le d})\right)T} + d\,\log T \right)~.
\end{align*}
\end{theorem}
%
%
\begin{remark}
Theorem~\ref{th:main} shows a natural trade-off between delay and information. To make it clear, suppose $N \approx K$. In this case, the regret bound becomes of order $\sqrt{\bigl(d + \alpha(G_{\le d})\bigr)T\ln K} + d\ln T$. Now, if $d$ is as big as the diameter $d_G$ of $G$,
then $\alpha(G_{\le d})=1$. This means that at every time step all $N \approx K$ agents observe (with some delay) the losses of each other's actions. This is very much reminiscent of a full information scenario, and in fact our bound becomes of order $\sqrt{(d_G+1)T\ln K} + d_G\ln T$, which is close to the full information minimax rate $\sqrt{(d+1)T\ln K}$ when feedback has a constant delay $d$~\citep{weinberger2002delayed}. When $G$ is sparse (i.e., $d_G$ is likely to be large, say $d_G \approx N$), then agents have no advantage in taking $d = d_G$ since $d_G \approx N \approx K$. In this case, agents may even give up cooperation (choosing $d = 0$ in Figure~\ref{f:exp3-coop}),
and fall back on the standard bandit bound $\sqrt{TK\ln K}$, which corresponds to running {\sc Exp3-Coop} on the edgeless graph $G_0$. (No doubling trick is needed in this case, hence no extra $\log T$ term appears.)
\end{remark}
\begin{remark}
When $d = d_G$, each neighborhood $N_{\le d}(v)$ used in the loss estimate~(\ref{eq:estimator}) is equal to $V$, hence all agents receive the same feedback. Because they all start off from the same initial weights, the agents end up computing the same updates. This in turn implies that: (1) the individual regret incurred by each agent is the same as the average welfare regret $R_T^{\mathrm{coop}}$; (2) the messages exchanged by the agents (see Figure~\ref{f:protocol}) may be shortened by dropping the distribution part $\bp_{t-s}(v')$.
\end{remark}
\begin{remark}
\label{rem:choiceofd}
An interesting question is whether the agents can come up with a reasonable choice for the value of $d$ even when they lack any information whatsoever about the global structure of $G$. A partial answer to this question follows. It is easy to show that the choice $d = \sqrt{K}$ in Theorem~\ref{th:main} yields a bound on the average welfare regret of the form $K^{1/4}\sqrt{T\ln K} + \sqrt{K}(\ln T)$ {\em for all} $G$ (and irrespective to the value of $N = |V|$), provided $G$ is connected. This holds because, for any connected graph $G$, the independence number $\alpha(G_{\le d})$ is always bounded by\footnote
{
Because it holds for a worst-case (connected) $G$, this upper bound on $\alpha(G_{\le d})$ can be made tighter when specific graph topologies are considered.
}
$\bigl\lceil 2N\big/(d+2)\bigr\rceil$. To see why this latter statement is true, observe that the neighborhood $N_{\leq d/2}(v)$ of any node $v$ in $G_{\le d/2}$ contains at least $d/2+1$ nodes (including $v$), and any pair of nodes $v', v'' \in N_{\leq d/2}(v)$ are adjacent in $G_{\le d}$. Therefore, no independent set of $G_{\le d}$ can have size bigger than $\lceil 2N\big/(d+2)\bigr\rceil$. A more detailed bound is contained, e.g., in~\citep{fh97}.
\end{remark}


\section{Extensions: Cooperation with Individual Parameters}\label{s:many-d}
In this section, we analyze a modification of \textsc{Exp3-Coop} that allows each agent $v$ in the network to use a delay parameter $d(v)$ different from that of the other agents. We then show how such individual delays may improve the average welfare regret of the agents.
In the previous setting, where all agents use the same delay parameter $d$, messages have an implicit time-to-live equal to $d$. In this setting, however, agents may not have a detailed knowledge of the delay parameters used by the other agents. For this reason we allow an agent $v$ to generate messages with a time-to-live $ttl(v)$ possibly different from the delay parameter $d(v)$.
Note that the role of the two parameters $d(v)$ and $ttl(v)$ is inherently different. Whereas $d(v)$ rules the extent to which $v$ uses the messages received from the other agents, $ttl(v)$ limits the number of times a message from $v$ is forwarded to the other agents, thereby limiting the message complexity of the algorithm. In order to accomodate this additional parameter, we are required to modify the cooperative bandit protocol of Figure~\ref{f:protocol}. As in Section~\ref{s:multi}, we have an undirected communication network $G = (V,E)$ over the agents. However, in this new protocol the message that at the end of round $t$ each agent $v$ sends to his neighbors in $G$ has the format
\[
    m_t(v) = \Bigl\langle t,v,ttl(v),I_t(v),\loss_t\bigl(I_t(v)\bigr),\bp_t(v) \Bigr\rangle
\]
where $ttl(v)$ is the time-to-live parameter of agent $v$. Each message $m_{t-s}(v')$, which $v$ receives from a neighbor, first has its time-to-leave decremented by one. If the resulting value is positive, the message is forwarded to the other neighbors, otherwise it is dropped. Moreover, $v$ uses this message to update $\bp_t(v)$ only if $s \leq d(v)$.
Hence, at time $t$ an agent $v$ uses the message sent at time $t-s$ by $v'$ if and only if $\dist_G(v',v) = s$ with $s \le \min\{d(v),ttl(v')\}$, where $\dist_G(v,v')$ is
the shortest-path distance from $v'$ to $v$ in $G$.

Based on the collection ${\scP} = \{d(v),ttl(v)\}_{v\in V}$ of individual parameters, we define the directed graph $G_{\scP} = (V,E_{\scP})$ as follows: arc $(v',v) \in E_{\scP}$ if and only if $\dist_G(v,v') \leq \min\{d(v),ttl(v')\}$.
The in-neighborhood $N^-_{\scP}(v)$ of $v$ thus contains the set of all $v'\in V$ whose distance from $v$ is not larger than $\min\{d(v),ttl(v')\}$. Notice that, with this definition, $v\in N^-_{\scP}(v)$, so that $(V,E_{\scP})$ includes all self-loops $(v,v)$.
%
Figure~\ref{f:multi-d}(a) illustrates these concepts through a simple pictorial example.
\begin{remark}
\label{r:exo}
It is important to remark that the communication structure encoded by $\scP$ is an exogenous parameter of the regret minimization problem, and so our algorithms cannot trade it off against regret. In addition to that, the parameterization $\scP = \{d(v),ttl(v)\}_{v\in V}$ defines a simple and static communication graph which makes it relatively easy to express regret as a function of the amount of available communication. This would not be possible if we had each individual node $v$ decide whether to forward a message based, say, on its own local delay parameter $d(v)$. To see why, consider the situation where nodes $v$ and $v'$ are along the route of a message that is reaching $v$ before $v'$. The decision of $v$ to drop the message may clash with the willingness of $v'$ to receive it, and this may clearly happen when $d(v) < d(v')$. The structure of the communication graph resulting from this individual behavior of the nodes would be rather complicated. On the contrary, the time-to-live-based parametrization, which is commonly used in communication networks to control communication complexity, does not have this issue.
\end{remark}

\begin{figure}[t!]
\begin{picture}(-30,270)(-30,270)
\scalebox{0.65}{\includegraphics{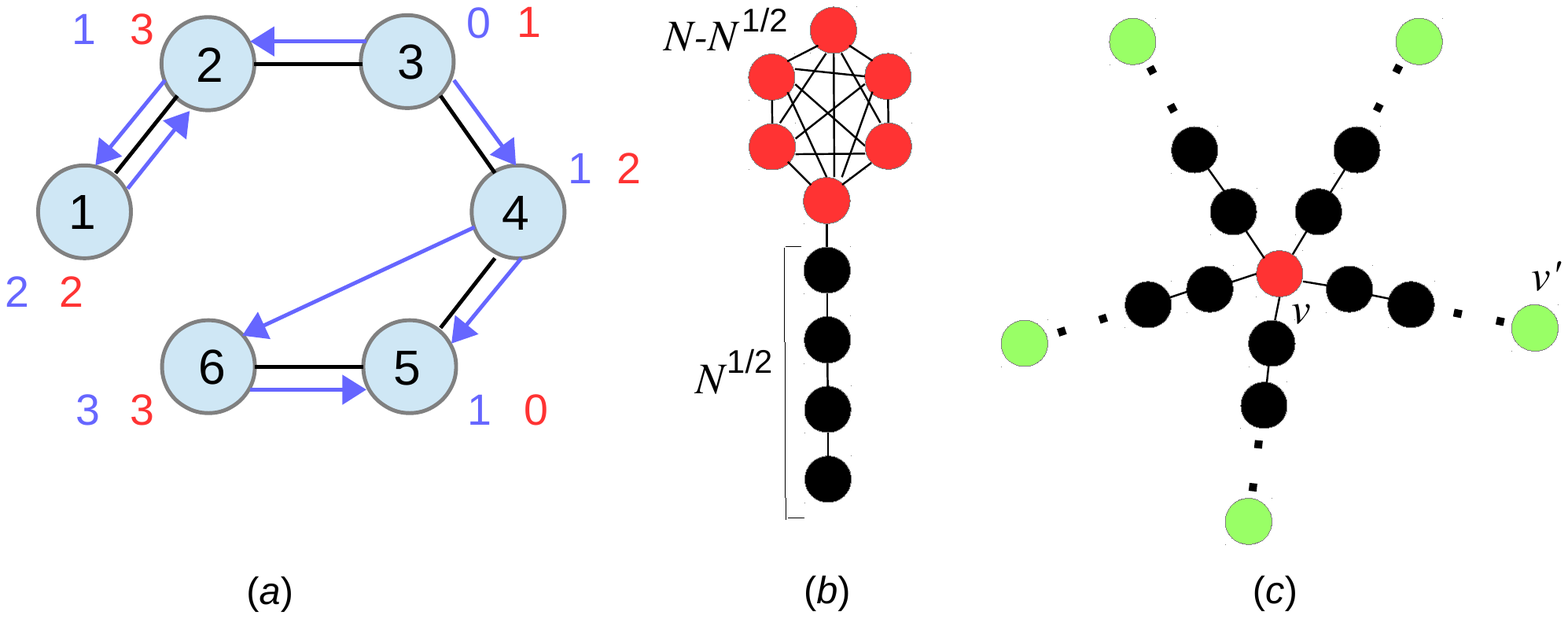}}
\end{picture}
\vspace{-1.8in}
\caption{\label{f:multi-d}
{\bf (a)} In this example, the communication network $G$ is an undirected line graph with $N =6$ agents, whose edges are depicted in black. Close to each node $v = 1, \ldots, 6$ is the individual delay $d(v)$ (in blue), and the individual time-to-leave $ttl(v)$ (in red). The arcs (a.k.a., directed edges) of the induced directed graph $G_{\scP}$ are also depicted in blue. Self-loops are not depicted. For instance, we have $N^-_{\scP}(5) =\{4,5,6\}$ and $N^-_{\scP}(3) =\{3\}$.
{\bf(b)} A communication network having a dense (red nodes) and a sparse (black nodes) region. The black region has $N^{1/2}$ agents, the red one has $N-N^{1/2}$ agents.
{\bf (c)} A star graph with long rays. The center $v$ (in red) sets a small $d(v)$ and a large $ttl(v)$. The peripheral nodes $v'$ (in green) set a large $d(v')$ and a small $ttl(v')$.
}
\end{figure}
Figure~\ref{f:exp3-coop2} contains our algorithm (called {\sc Exp3-Coop2}) for this setting. {\sc Exp3-Coop2} is a strict generalization of {\sc Exp3-Coop}, and so is its analysis. The main difference between the two algorithms is that {\sc Exp3-Coop2} deals with directed graphs. This fact prevents us from using the same techniques of Section~\ref{s:single-d} in order to control the regret. Intuitively, adding orientations to the edges reduces the information available to the agents and thus increases the variance of their loss estimates. Thus, in order to control this variance, we need a lower bound\footnote
{
We find it convenient to derive this lower bound without mixing with the uniform distribution over $A$ ---see, e.g., \citep{auer2002nonstochastic}--- but in a slightly different manner. This facilitates our delayed feedback analysis.
}
on the probabilities $p_t(i,v)$.
\begin{figure}[t!]
\bookbox{
\textbf{The Exp3-Coop2 Algorithm}\\
\textbf{Parameters:} Undirected graph $G = (V,E)$; learning rate $\eta$; exploration parameter $\delta > 0.$\\
\textbf{Init:} Each $v \in V$ sets weights $w_1(i,v) = 1$, for all $i \in A$, delay $d(v)$, and time-to-live $ttl(v)$.\\
\textbf{For $t=1,2,\dots$}
\begin{enumerate}
\item Each agent $v \in V$ plays action $I_t(v) \in A$ drawn according to distribution
$
    \bp_t(v) = (p_t(1,v), \ldots, p_t(K,v))~,
$
where
\[
p_t(i,v) = \frac{\tprob_t(i,v)}{\TProb_t(v)}, \qquad \TProb_t(v) = \sum_{j \in A} \tprob_t(j,v),
\]
and
\[
\tprob_t(i,v) = \max\left\{\frac{w_t(i,v)}{W_t(v)},\frac{\delta}{K}\right\}, \qquad W_t(v) = \sum_{j \in A} w_t(j,v)\,;
\]
\item Each agent $v \in V$ observes loss $\loss_t\bigl(I_t(v)\bigr)$ and exchanges messages with his neighbors (see main text for an explanation);
\item Each agent $v \in V$ performs the update $w_{t+1}(i,v) = p_t(i,v)\,\exp\bigl(-\eta\,\hloss_t(i,v)\bigr)$ for all $i \in A$,
where
%
\[
    \hloss_t(i,v) = \left\{ \begin{array}{cl}
        \displaystyle{\frac{\loss_{t-d(v)}(i)}{q_{\scP,t-d(v)}(i,v)}}\, B_{\scP,t-d(v)}(i,v) & \text{if $t > d(v)$,}
        \\
        0 & \text{otherwise,}
    \end{array} \right.
\]
%
and
$
B_{\scP,t-d(v)}(i,v) = \Ind{\exists v' \in N^-_{\scP}(v) \,:\, I_{t-d(v)}(v') = i}
$,
with
%
\[
    q_{\scP,t-d(v)}(i,v) = 1 - \prod_{v' \in N^-_{\scP}(v)}\bigl(1-p_{t-d(v)}(i,v')\bigr)~.
\]
\end{enumerate}
}
\caption{
The Exp3-Coop2 algorithm with individual delay and time-to-live parameters.
}
\label{f:exp3-coop2}
\end{figure}
From Figure~\ref{f:exp3-coop2}, one can easily see that
\begin{equation}\label{e:ptildebound}
1 = \sum_{i\in A}\frac{w_t(i,v)}{W_t(v)} \leq \TProb_t(v) \leq  \sum_{i\in A}\left( \frac{w_t(i,v)}{W_t(v)} + \frac{\delta}{K} \right) = 1 + \delta
\end{equation}
implying the lower bound
\(
p_t(i,v) \geq \frac{\delta}{K(1 + \delta)}\,,
\)
holding for all $i$, $t$, and $v$.


The following theorem
is the main result of this section.
\begin{theorem}\label{t:main-individual}
The regret of {\sc Exp3-Coop2} run over a network $G = (V,E)$ of $N$ agents, each agent $v$ using individual delay $d(v)$, individual time-to-leave $ttl(v)$, exploration parameter $\delta = 1/T$, and learning rate $\eta$ such that $\eta \rightarrow 0$ as $T \rightarrow \infty$ satisfies, when $T$ grows large,
\[
R_T^{\mathrm{coop}}
= \mathcal{O}\left(\frac{\ln K}{\eta} + \eta\Big({\bar d}_V + \frac{K}{N}\,\alpha\left(G_{\scP}\right)\,\ln(T N K)\Big)T\right)~,
    \quad\text{where} \quad {\bar d}_V = \frac{1}{N}\,\sum_{v\in V} d(v)~.
\]
\end{theorem}
Using a doubling trick in much the same way we used it to prove Theorem~\ref{th:main}, we can prove the following result.
\begin{corollary}\label{c:main-individual-tuned}
The regret of {\sc Exp3-Coop2} run over a network $G = (V,E)$ of $N$ agents, each agent $v$ using individual delay $d(v)$, individual time-to-leave $ttl(v)$, exploration parameter $\delta = 1/T$, and individual learning rate $\eta(v)$ adaptively selected by each agent through a doubling trick, satisfies, when $T$ grows large
\[
    R_T^{\mathrm{coop}}
=
    \mathcal{O}\left(\sqrt{(\ln K)\left({\bar d}_V + 1 + \frac{K}{N}\,\alpha(G_{\scP})\ln(TNK)\right)T} + {\bar d}_V\,\big(\ln T + \ln\ln(TNK)\big) \right)~.
\]
\end{corollary}
To illustrate the advantage of having individual delays as opposed to sharing the same delay value, it suffices to consider a communication network including regions of different density. Concretely, consider the graph in Figure~\ref{f:multi-d}(b) with a large densely connected region (red agents) and a small sparsely connected region black agents). In this example, the black agents prefer a large value of their individual delay so as to receive more information from nearby agents, but this comes at the price of a larger bias for their estimators $\hloss_t(i,v)$. On the contrary, information from nearby agents is readily available to the red agents, so that they do not gain any regret improvement from a large delay parameter.
A similar argument applies here to the individual time-to-live values: red agents $v$ will set a small $ttl(v)$ to reduce communication. Black agents $v'$ may decide to set $ttl(v')$ depending on their intention to reach the red nodes. But because the red agents have set a small $d(v)$, any effort made by $v'$ trying to reach them would be a communication waste. Hence, it is reasonable for a black agent $v'$ to set a moderately large value for $ttl(v')$, but perhaps not so large as to reach the red agents. One can read this off the bounds in both Theorem~\ref{t:main-individual} and Corollary~\ref{c:main-individual-tuned}, as explained next. Suppose for simplicity that $K \approx N$ so that, disregarding log factors, these bounds depend on parameters $\scP$ only through the quantity
$
H = {\bar d}_V + \alpha\left(G_{\scP}\right)
$.
Now, in the case of a common delay parameter $d$ (Section~\ref{s:single-d}), it is not hard to see that the best setting for $d$ in order to minimize $H$ is of the form $d = N^{1/4}$, resulting in $H = \Theta(N^{1/4})$. On the other hand, the best setting for the individual delays is $d(v) = 1$ when $v$ is red, and $d(v) = \sqrt{N}$ when $v$ is black, resulting in $H = \Theta(1)$.

The time-to-live parameters $ttl(v)$ affect the regret bound only through $\alpha\left(G_{\scP}\right)$, but they clearly play the additional role of bounding the message complexity of the algorithm. In our example of Figure \ref{f:multi-d}(b), we essentially have $d(v) \approx ttl(v)$ for all $v$. A typical scenario where agents may have $d(v) \neq ttl(v)$ is illustrated in Figure~\ref{f:multi-d}(c). In this case, we have star-like graph where a central agent is connected through long rays to all others agents. The center $v$ prefers to set a small $d(v)$, since it has a large degree, but also a large $ttl(v)$ in order to reach the green peripheral nodes. The green nodes $v'$ are reasonably doing the opposite: a large $d(v')$ in order to gather information from other nodes, but also a smaller time-to-live than the center, for the information transmitted by $v'$ is comparatively less valuable to the whole network than the one transmitted by the center.

Agents can set their individual parameters in a topology-dependent manner using any algorithm for assessing the centrality of nodes in a distributed fashion ---e.g.,~\citep{wz13}, and references therein. This can be done at the beginning in a number of rounds which only depends on the network topology (but not on $T$). Hence, this initial phase would affect the regret bound only by an additive constant.

\section{Delayed Losses (for a Single Agent)}\label{s:delayed}
\textsc{Exp3-Coop} can be specialized to the setting where a single agent is facing a bandit problem in which the loss of the chosen action is observed with a fixed delay $d$. In this setting, at the end of each round $t$ the agent incurs loss $\loss_t(I_t)$ and observes $\loss_{t-d}(I_{t-d})$, if $t > d$, and nothing otherwise. The regret is defined in the usual way,
\[
    R_T = \E\left[\sum_{t=1}^T \loss_t(I_t)\right] - \min_{i=1,\dots,K} \sum_{t=1}^T \loss_t(i)~.
\]
This problem was studied by~\cite{weinberger2002delayed} in the full information case, for which they proved that $\sqrt{(d+1)T\ln K}$ is the optimal order for the minimax regret. The result was extended to the bandit case by~\cite{neu2010online,neu2014online} ---see also~\cite{joulani2013online}--- whose techniques can be used to obtain a regret bound of order $\sqrt{(d+1)KT}$. Yet, no matching lower bound was available for the bandit case.

As a matter of fact, the upper bound $\sqrt{(d+1)KT}$ for the bandit case is easily obtained: just run in parallel $d+1$ instances of the minimax optimal bandit algorithm for the standard (no delay) setting, achieving $R_T \le \sqrt{KT}$ (ignoring constant factors). At each time step $t = (d+1)r + s$ (for $r=0,1,\dots$ and $s=0,\dots,d$), use instance $s+1$ for the current play. Hence, the no-delay bound applies to every instance and, assuming $d+1$ divides $T$, we immediately obtain
%
\(
    R_T \le \sum_{s=1}^{d+1} \sqrt{K\frac{T}{d+1}} \le \sqrt{(d+1)KT}~,
\)
%
again, ignoring constant factors.

Next, we show that the machinery we developed in Section \ref{s:single-d} delivers an improved upper bound on the regret for the bandit problem with delayed losses, and then we complement this result by providing a lower bound matching the upper bound up to log factors, thereby characterizing (up to log factors) the minimax regret for this problem.
\begin{corollary}\label{c:delayed}
In the nonstochastic bandit setting with $K \ge 2$ actions and delay $d \ge 0$, where at the end of each round $t$ the predictor has access to the losses $\loss_1(I_1),\dots,\loss_s(I_s)\in [0,1]^K$ for $s = \max\{1,t-d\}$, the minimax regret is of order
$
    \sqrt{(K + d)T}~,
$
ignoring logarithmic factors.
\end{corollary}
%

%

\section{Conclusions and Ongoing Research}\label{s:conc}
We have investigated a cooperative and nonstochastic bandit scenario where cooperation comes at the price of delayed information. We have proven average welfare regret bounds that exhibit a natural tradeoff between amount cooperation and delay, the tradeoff being ruled by the underlying communication network topology. As a by-product of our analysis, we have also provided the first characterization to date of the regret of learning with (constant) delayed feedback in an adversarial bandit setting.
There are a number of possible extensions which we are currently considering:
\begin{enumerate}
\item So far our analysis only delivers average welfare regret bounds. It would be interesting to show simultaneous regret bounds that hold for each agent individually. We conjecture that the individual regret bound of an agent $v$ is of the form $\sqrt{(\ln K)\left(d+\frac{K}{|N_{\le d}(v)|}\right)\,T}$, where $|N_{\le d}(v)|$ is the degree of $v$ in $G_{\le d}$ (plus one). Such bound would in fact imply, e.g., the one in Theorem~\ref{th:main}. A possible line of attack to solve this problem could be the use of graph sparsity along the lines of~\citep{NIPS2015_5814,NIPS2013_4939,mania2015perturbed,NIPS2014_5242}.
%
\item It would be nice to characherize the average welfare regret by complementing our upper bounds with suitable {\em lower} bounds: Is the upper bound of Theorem~\ref{th:main} optimal in the communication model considered here?
%
\item The two algorithms we designed do not use the loss information in the most effective way, for they both postpone the update step by $d$ (Figure \ref{f:exp3-coop}) or $d(v)$ ((Figure \ref{f:exp3-coop2}) time steps. In fact, we do have generalized versions of both algorithms where all losses $\ell_{t-s}(i)$ coming from agents at distance $s$ from any given agent $v$ are indeed used at time $t$ by agent $v$ i.e., as soon as these losses become available to $v$.
The resulting regret bounds mix delays and independence numbers of graphs at different levels of delay. (Details will be given in the full version of this paper.)
More ambitiously, it is natural to think of ways to adaptively tune our algorithms so as to automatically determine the best delay parameter $d$. For instance, disregarding message complexity, is there a way for each agent to adaptively tune $d$ locally so to minimize the bound in Theorem~\ref{th:main}?
\item Our messages $m_t(v)$ contain both action/loss information and distribution information. Is it possible to drop the distribution information and still achieve average welfare regret bounds similar to those in Theorems~\ref{th:nontuned} and~\ref{th:main}?
\item Even for the single-agent setting, we do not know whether regret bounds of the form $\sqrt{(D+T)\ln K}$, where $D$ is the total delay experienced over the $T$ rounds, could be proven ---see~\citep{DBLP:conf/aaai/JoulaniGS16,NIPS2015_5833} for similar results in the full-information setting. In general, the study of learning on a communication network with time-varying delays, and its impact on the regret rates, is a topic which is certainly worth of attention.
\end{enumerate}

\subsection*{Acknowledgments}
We thank the anonymous reviewers for their careful reading, and for their thoughtful suggestions that greatly improved the presentation of this paper.
Yishay Mansour is supported in part by the Israeli Centers of Research Excellence (I-CORE) program, (Center No.~4/11), by a grant from the Israel Science Foundation (ISF), by a grant from United States-Israel Binational Science Foundation (BSF) and by a grant from the Len Blavatnik and the Blavatnik Family Foundation.

\bibliography{cesa-bianchi16}

\begin{thebibliography}{38}
\providecommand{\natexlab}[1]{#1}
\providecommand{\url}[1]{\texttt{#1}}
\expandafter\ifx\csname urlstyle\endcsname\relax
  \providecommand{\doi}[1]{doi: #1}\else
  \providecommand{\doi}{doi: \begingroup \urlstyle{rm}\Url}\fi

\bibitem[Agarwal and Duchi(2011)]{NIPS2011_4247}
Alekh Agarwal and John~C Duchi.
\newblock Distributed delayed stochastic optimization.
\newblock In J.~Shawe-Taylor, R.S. Zemel, P.L. Bartlett, F.~Pereira, and K.Q.
  Weinberger, editors, \emph{Advances in Neural Information Processing Systems
  24}, pages 873--881. Curran Associates, Inc., 2011.

\bibitem[Alon et~al.(2014)Alon, Cesa-Bianchi, Gentile, Mannor, Mansour, and
  Shamir]{alon2014nonstochastic}
Noga Alon, Nicol{\`o} Cesa-Bianchi, Claudio Gentile, Shie Mannor, Yishay
  Mansour, and Ohad Shamir.
\newblock Nonstochastic multi-armed bandits with graph-structured feedback.
\newblock \emph{arXiv preprint arXiv:1409.8428}, 2014.

\bibitem[Auer et~al.(2002)Auer, Cesa-Bianchi, Freund, and
  Schapire]{auer2002nonstochastic}
Peter Auer, Nicol{\`o} Cesa-Bianchi, Yoav Freund, and Robert~E Schapire.
\newblock The nonstochastic multiarmed bandit problem.
\newblock \emph{SIAM Journal on Computing}, 32\penalty0 (1):\penalty0 48--77,
  2002.

\bibitem[Awerbuch and Kleinberg(2008)]{awerbuch2008competitive}
Baruch Awerbuch and Robert Kleinberg.
\newblock Competitive collaborative learning.
\newblock \emph{Journal of Computer and System Sciences}, 74\penalty0
  (8):\penalty0 1271--1288, 2008.

\bibitem[Barrett and Stone(2011)]{barrett2011ad}
Samuel Barrett and Peter Stone.
\newblock Ad hoc teamwork modeled with multi-armed bandits: An extension to
  discounted infinite rewards.
\newblock In \emph{Proceedings of 2011 AAMAS Workshop on Adaptive and Learning
  Agents}, pages 9--14, 2011.

\bibitem[Cesa-Bianchi et~al.(2016)Cesa-Bianchi, Gentile, Mansour, and
  Minora]{cesa2016delay}
Nicolo' Cesa-Bianchi, Claudio Gentile, Yishay Mansour, and Alberto Minora.
\newblock Delay and cooperation in nonstochastic bandits.
\newblock \emph{arXiv preprint, arXiv:1602.04741v2}, 2016.

\bibitem[Duchi et~al.(2013)Duchi, Jordan, and McMahan]{NIPS2013_4939}
John Duchi, Michael~I Jordan, and Brendan McMahan.
\newblock Estimation, optimization, and parallelism when data is sparse.
\newblock In C.~J.~C. Burges, L.~Bottou, M.~Welling, Z.~Ghahramani, and K.~Q.
  Weinberger, editors, \emph{Advances in Neural Information Processing Systems
  26}, pages 2832--2840. Curran Associates, Inc., 2013.

\bibitem[Duchi et~al.(2015)Duchi, Chaturapruek, and
  R{\'e}]{duchi2015asynchronous}
John~C Duchi, Sorathan Chaturapruek, and Christopher R{\'e}.
\newblock Asynchronous stochastic convex optimization.
\newblock \emph{arXiv preprint arXiv:1508.00882}, 2015.

\bibitem[Dud{\'{\i}}k et~al.(2011)Dud{\'{\i}}k, Hsu, Kale, Karampatziakis,
  Langford, Reyzin, and Zhang]{DBLP:conf/uai/DudikHKKLRZ11}
Miroslav Dud{\'{\i}}k, Daniel~J. Hsu, Satyen Kale, Nikos Karampatziakis, John
  Langford, Lev Reyzin, and Tong Zhang.
\newblock Efficient optimal learning for contextual bandits.
\newblock In \emph{{UAI} 2011, Proceedings of the Twenty-Seventh Conference on
  Uncertainty in Artificial Intelligence, Barcelona, Spain, July 14-17, 2011},
  pages 169--178, 2011.

\bibitem[Firby and Haviland(1997)]{fh97}
P.~Firby and J.~Haviland.
\newblock Independence and average distance in graphs.
\newblock \emph{Discrete Applied Mathematics}, 75:\penalty0 27--37, 1997.

\bibitem[Joulani et~al.(2013)Joulani, Gy{\"o}rgy, and
  Szepesv{\'a}ri]{joulani2013online}
Pooria Joulani, Andr{\'a}s Gy{\"o}rgy, and Csaba Szepesv{\'a}ri.
\newblock Online learning under delayed feedback.
\newblock In \emph{Proceedings of the 30th International Conference on Machine
  Learning (ICML-13)}, pages 1453--1461, 2013.

\bibitem[Joulani et~al.(2016)Joulani, Gy{\"{o}}rgy, and
  Szepesv{\'{a}}ri]{DBLP:conf/aaai/JoulaniGS16}
Pooria Joulani, Andr{\'{a}}s Gy{\"{o}}rgy, and Csaba Szepesv{\'{a}}ri.
\newblock Delay-tolerant online convex optimization: Unified analysis and
  adaptive-gradient algorithms.
\newblock In \emph{Proceedings of the Thirtieth {AAAI} Conference on Artificial
  Intelligence, February 12-17, 2016, Phoenix, Arizona, {USA.}}, pages
  1744--1750, 2016.

\bibitem[Kar et~al.(2011)Kar, Poor, and Cui]{kar2011bandit}
Soummya Kar, H~Vincent Poor, and Shuguang Cui.
\newblock Bandit problems in networks: Asymptotically efficient distributed
  allocation rules.
\newblock In \emph{50th IEEE Conference on Decision and Control and European
  Control Conference (CDC-ECC)}, pages 1771--1778. IEEE, 2011.

\bibitem[Kleinberg et~al.(2009)Kleinberg, Piliouras, and
  Tardos]{kleinberg2009multiplicative}
Robert Kleinberg, Georgios Piliouras, and {\'E}va Tardos.
\newblock Multiplicative updates outperform generic no-regret learning in
  congestion games.
\newblock In \emph{Proceedings of the forty-first annual ACM symposium on
  Theory of computing}, pages 533--542. ACM, 2009.

\bibitem[Koc\'{a}k et~al.(2014)Koc\'{a}k, Neu, Valko, and Munos]{k+14}
Tom\'{a}\v{s} Koc\'{a}k, Gergely Neu, Michal Valko, and Remi Munos.
\newblock Efficient learning by implicit exploration in bandit problems with
  side observations.
\newblock In \emph{Advances in Neural Information Processing Systems 27}, pages
  613--621. 2014.

\bibitem[Landgren et~al.(2015)Landgren, Srivastava, and
  Leonard]{landgren2015distributed}
Peter Landgren, Vaibhav Srivastava, and Naomi~Ehrich Leonard.
\newblock On distributed cooperative decision-making in multiarmed bandits.
\newblock \emph{arXiv preprint arXiv:1512.06888}, 2015.

\bibitem[Li et~al.(2013)Li, Andersen, and Smola]{li2013distributed}
Mu~Li, David~G. Andersen, and Alexander Smola.
\newblock Distributed delayed proximal gradient methods.
\newblock In \emph{NIPS Workshop on Optimization for Machine Learning}, 2013.

\bibitem[Linial(1992)]{Linial92}
Nathan Linial.
\newblock Locality in distributed graph algorithms.
\newblock \emph{{SIAM} J. Comput.}, 21\penalty0 (1):\penalty0 193--201, 1992.

\bibitem[Liu et~al.(2015)Liu, Wright, R{\'e}, Bittorf, and
  Sridhar]{liu2015asynchronous}
Ji~Liu, Stephen~J Wright, Christopher R{\'e}, Victor Bittorf, and Srikrishna
  Sridhar.
\newblock An asynchronous parallel stochastic coordinate descent algorithm.
\newblock \emph{The Journal of Machine Learning Research}, 16\penalty0
  (1):\penalty0 285--322, 2015.

\bibitem[Mania et~al.(2015)Mania, Pan, Papailiopoulos, Recht, Ramchandran, and
  Jordan]{mania2015perturbed}
Horia Mania, Xinghao Pan, Dimitris Papailiopoulos, Benjamin Recht, Kannan
  Ramchandran, and Michael~I Jordan.
\newblock Perturbed iterate analysis for asynchronous stochastic optimization.
\newblock \emph{arXiv preprint arXiv:1507.06970}, 2015.

\bibitem[McMahan and Streeter(2014)]{NIPS2014_5242}
Brendan McMahan and Matthew Streeter.
\newblock Delay-tolerant algorithms for asynchronous distributed online
  learning.
\newblock In Z.~Ghahramani, M.~Welling, C.~Cortes, N.D. Lawrence, and K.Q.
  Weinberger, editors, \emph{Advances in Neural Information Processing Systems
  27}, pages 2915--2923. Curran Associates, Inc., 2014.

\bibitem[Mesterharm(2005)]{mesterharm2005line}
Chris Mesterharm.
\newblock On-line learning with delayed label feedback.
\newblock In \emph{Algorithmic Learning Theory}, pages 399--413. Springer,
  2005.

\bibitem[Mesterharm(2007)]{Mester2007}
Chris Mesterharm.
\newblock \emph{{Improving Online Learning}}.
\newblock PhD thesis, Department of Computer Science, Rutgers University, 2007.

\bibitem[Neu(2015)]{neu15}
Gergely Neu.
\newblock Explore no more: Improved high-probability regret bounds for
  non-stochastic bandits.
\newblock In \emph{Advances in Neural Information Processing Systems 28
  (NIPS)}, 2015.

\bibitem[Neu et~al.(2010)Neu, Antos, Gy\"{o}rgy, and
  Szepesv\'{a}ri]{neu2010online}
Gergely Neu, Andras Antos, Andr\'{a}s Gy\"{o}rgy, and Csaba Szepesv\'{a}ri.
\newblock Online {M}arkov decision processes under bandit feedback.
\newblock In \emph{Advances in Neural Information Processing Systems 23}, pages
  1804--1812. Curran Associates, Inc., 2010.

\bibitem[Neu et~al.(2014)Neu, Gyorgy, Szepesvari, and Antos]{neu2014online}
Gergely Neu, Andras Gyorgy, Csaba Szepesvari, and Andras Antos.
\newblock Online markov decision processes under bandit feedback.
\newblock \emph{Automatic Control, IEEE Transactions on}, 59\penalty0
  (3):\penalty0 676--691, 2014.

\bibitem[Pan et~al.(2015)Pan, Papailiopoulos, Oymak, Recht, Ramchandran, and
  Jordan]{NIPS2015_5814}
Xinghao Pan, Dimitris Papailiopoulos, Samet Oymak, Benjamin Recht, Kannan
  Ramchandran, and Michael~I Jordan.
\newblock Parallel correlation clustering on big graphs.
\newblock In C.~Cortes, N.~D. Lawrence, D.~D. Lee, M.~Sugiyama, and R.~Garnett,
  editors, \emph{Advances in Neural Information Processing Systems 28}, pages
  82--90. Curran Associates, Inc., 2015.

\bibitem[Quanrud and Khashabi(2015)]{NIPS2015_5833}
Kent Quanrud and Daniel Khashabi.
\newblock Online learning with adversarial delays.
\newblock In C.~Cortes, N.D. Lawrence, D.D. Lee, M.~Sugiyama, and R.~Garnett,
  editors, \emph{Advances in Neural Information Processing Systems 28}, pages
  1270--1278. Curran Associates, Inc., 2015.

\bibitem[Rosenski et~al.(2015)Rosenski, Shamir, and Szlak]{RosenskiSS15}
Jonathan Rosenski, Ohad Shamir, and Liran Szlak.
\newblock Multi-player bandits -- a musical chairs approach.
\newblock \emph{CoRR}, abs/1512.02866, 2015.

\bibitem[Seldin et~al.(2014)Seldin, Bartlett, Crammer, and
  Abbasi-Yadkori]{seldin2014prediction}
Yevgeny Seldin, Peter Bartlett, Koby Crammer, and Yasin Abbasi-Yadkori.
\newblock Prediction with limited advice and multiarmed bandits with paid
  observations.
\newblock In \emph{Proceedings of The 31st International Conference on Machine
  Learning}, pages 280--287, 2014.

\bibitem[Stranders et~al.(2012)Stranders, Tran-Thanh, Fave, Rogers, and
  Jennings]{stranders2012dcops}
Ruben Stranders, Long Tran-Thanh, Francesco M~Delle Fave, Alex Rogers, and
  Nicholas~R Jennings.
\newblock Dcops and bandits: Exploration and exploitation in decentralised
  coordination.
\newblock In \emph{Proceedings of the 11th International Conference on
  Autonomous Agents and Multiagent Systems-Volume 1}, pages 289--296.
  International Foundation for Autonomous Agents and Multiagent Systems, 2012.

\bibitem[Suomela(2013)]{Suomela13}
Jukka Suomela.
\newblock Survey of local algorithms.
\newblock \emph{{ACM} Computing Surveys}, 45\penalty0 (2):\penalty0 24, 2013.

\bibitem[Szorenyi et~al.(2013)Szorenyi, Busa-Fekete, Heged{\"u}s, Orm{\'a}ndi,
  Jelasity, and K{\'e}gl]{szorenyi2013gossip}
Balazs Szorenyi, R{\'o}bert Busa-Fekete, Istv{\'a}n Heged{\"u}s, R{\'o}bert
  Orm{\'a}ndi, M{\'a}rk Jelasity, and Bal{\'a}zs K{\'e}gl.
\newblock Gossip-based distributed stochastic bandit algorithms.
\newblock In \emph{30th International Conference on Machine Learning (ICML
  2013)}, volume~28, pages 19--27. ACM Press, 2013.

\bibitem[Tekin and van~der Schaar(2015)]{TekinS15}
Cem Tekin and Mihaela van~der Schaar.
\newblock Distributed online learning via cooperative contextual bandits.
\newblock \emph{{IEEE} Transactions on Signal Processing}, 63\penalty0
  (14):\penalty0 3700--3714, 2015.

\bibitem[Tekin et~al.(2014)Tekin, Zhang, and van~der Schaar]{TekinZS14}
Cem Tekin, Simpson~Z. Zhang, and Mihaela van~der Schaar.
\newblock Distributed online learning in social recommender systems.
\newblock \emph{J. Sel. Topics Signal Processing}, 8\penalty0 (4):\penalty0
  638--652, 2014.

\bibitem[Wehmuth and Ziviani(2013)]{wz13}
Klaus Wehmuth and Artur Ziviani.
\newblock Daccer: Distributed assessment of the closeness centrality ranking in
  complex networks.
\newblock \emph{Computer Networks}, 57\penalty0 (13):\penalty0 2536--2548,
  2013.

\bibitem[Weinberger and Ordentlich(2002)]{weinberger2002delayed}
Marcelo~J Weinberger and Erik Ordentlich.
\newblock On delayed prediction of individual sequences.
\newblock \emph{IEEE Transactions on Information Theory}, 48\penalty0
  (7):\penalty0 1959--1976, 2002.

\bibitem[Zinkevich et~al.(2009)Zinkevich, Langford, and Smola]{NIPS2009_3888}
Martin Zinkevich, John Langford, and Alex~J. Smola.
\newblock Slow learners are fast.
\newblock In Y.~Bengio, D.~Schuurmans, J.D. Lafferty, C.K.I. Williams, and
  A.~Culotta, editors, \emph{Advances in Neural Information Processing Systems
  22}, pages 2331--2339. Curran Associates, Inc., 2009.

\end{thebibliography}

\appendix


\section{Proofs from Section \ref{s:prel}}\label{a:prel}

\proofof{Lemma \ref{l:sandwich}}

\begin{proof}
Directly from the definition of the update~(\ref{eq:exp-upd}), $w_{t+1}(i) \le p_t(i)$ for all $i \in A$, so that $W_{t+1} \le 1$, which in turn implies $w_{t+1}(i) \le w_{t+1}(i)/W_{t+1} = p_{t+1}(i)$. Therefore
\begin{align*}
    p_{t+1}(i) - p_t(i)
&\ge
    w_{t+1}(i) - p_t(i)\\
&=
    p_t(i)\left(e^{-\eta\,\hloss_t(i)} - 1\right)\\
&\ge
    -\eta\,p_t(i)\hloss_t(i)~,
\end{align*}
the last inequality using $1-e^{-x} \leq x$ for $x \geq 0$.
Similarly,
\begin{align*}
    p_{t+1}(i) - p_t(i)
&\le
    p_{t+1}(i) - w_{t+1}(i)\\
&=
    p_{t+1}(i) - p_{t+1}(i)W_{t+1}\\
&=
    p_{t+1}(i)\sum_{j \in A} \bigl(p_t(j) - w_{t+1}(j)\bigr)\\
&=
    p_{t+1}(i)\sum_{j \in A} p_t(j)\left(1 - e^{-\eta\,\hloss_t(j)}\right)\\
&\le
    \eta\,p_{t+1}(i)\sum_{j \in A} p_t(j)\hloss_t(j)
\end{align*}
concluding the proof.
\end{proof}

\proofof{Lemma \ref{l:mult}}
\begin{proof}
We proceed by induction over $t$. For all $t \le d$, $\hloss_t(\cdot) = 0$. Hence $p_t(\cdot) = 1/K$, and the lemma trivially holds. For $t > d$ we can write
\begin{align*}
    \sum_{i\in A} p_t(i)\hloss_t(i)
&=
    \sum_{i\in A} p_t(i)\frac{\loss_{t-d}(i)}{q_{t-d}(i)} B_{t-d}(i)
\\ &\le
    \sum_{i\in A} \frac{p_t(i)}{q_{t-d}(i)}
    \qquad\qquad\qquad\ \ \  \text{(because $B_{t-d}(i)\loss_{t-d}(i) \leq 1$)}
\\ &\le
    \sum_{i\in A} \left(1 + \frac{1}{d}\right)^d\frac{p_{t-d}(i)}{q_{t-d}(i)}
    \qquad \text{(by the inductive hypothesis)}
\\ &\le
    \left(1 + \frac{1}{d}\right)^d K
    \qquad\qquad\qquad \text{(because $q_{t-d}(i) \ge p_{t-d}(i)$)}
\\ &\le
    Ke~.
\end{align*}
Hence, using Lemma \ref{l:sandwich},
\[
    p_{t+1}(i)\bigl(1-\eta\,Ke\bigr)
\le
    p_{t+1}(i)\left(1-\eta\,\sum_{j \in A} p_t(j)\hloss_t(j)\right)
\le
    p_t(i)
\]
which implies
$
    p_{t+1}(i) \le \left(1 + \frac{1}{d}\right)p_t(i)
$
whenever $\eta \le \frac{1}{Ke(d+1)}$.
\end{proof}

\section{Proofs from Section \ref{s:single-d}}\label{a:single-d}

The next lemma relates the variance of the estimates~(\ref{eq:estimator}) to the structure of the communication graph $G$. The lemma is stated for a generic undirected communication graph $G$, but our application of it
actually involves graph $G_{\le d}$.
\begin{lemma}
\label{l:q-bound}
Let $G = (V,E)$ be an undirected graph with independence number $\alpha(G)$. For each $v \in V$, let $N_{\le 1}(v)$ be the neighborhood of node $v$ (including $v$ itself), and $\bp(v) = \bigl(p(1,v),\dots,p(K,v)\bigr)$ be a probability distribution over $A = \{1,\dots,K\}$. Then, for all $i \in A$,
\[
    \sum_{v \in V} \frac{p(i,v)}{q(i,v)}
\le
    \frac{1}{1-e^{-1}}\left(\alpha(G) + \sum_{v\in V} p(i,v) \right) \quad
\text{where} \quad q(i,v) = 1 - \prod_{v' \in N_{\le 1}(v)}\bigl(1-p(i,v')\bigr)~.
\]
\end{lemma}
\begin{proof}
Fix $i \in A$ and set for brevity $P(i,v) = \sum_{v' \in N_{\le 1}(v)} p(i,v')$. We can write
\begin{align*}
    \sum_{v \in V} \frac{p(i,v)}{q(i,v)}
&=
    \underbrace{\sum_{v \in V \,:\, P(i,v) \ge 1} \frac{p(i,v)}{q(i,v)}}_{\mathrm{(I)}}
\quad + \quad
    \underbrace{\sum_{v \in V\,:\, P(i,v) < 1} \frac{p(i,v)}{q(i,v)}}_{\mathrm{(II)}}~,
\end{align*}
and proceed by upper bounding the two terms~(I) and~(II) separately. Let $r(v)$ be the cardinality of $N_{\le 1}(v)$. We have, for any given $v \in V$,
\[
    \min\left\{  q(i,v) \,:\, \sum_{v' \in N_{\le 1}(v)} p(i,v') \ge 1 \right\}
=
    1-\left(1-\frac{1}{r(v)}\right)^{r(v)}
\ge
    1-e^{-1}~.
\]
The equality is due to the fact that the minimum is achieved when $p(i,v') = \frac{1}{r(v)}$ for all $v' \in N_{\le 1}(v)$, and the inequality comes from $r(v) \ge 1$ (for, $v \in N_{\le 1}(v)$). Hence
\begin{align*}
    \mathrm{(I)}
\le
    \sum_{v \in V \,:\, P(i,v) \ge 1} \frac{p(i,v)}{1-e^{-1}}
\le
    \sum_{v \in V} \frac{p(i,v)}{1-e^{-1}}~.
\end{align*}
As for~(II), using the inequality $1-x \leq e^{-x}, x\in [0,1]$, with $x = p(i,v')$, we can write
\begin{align*}
    q(i,v)
\ge
    1-\exp\left(-\sum_{v' \in N_{\le 1}(v)} p(i,v')\right)
=
    1-\exp\left(- P(i,v)\right)~.
\end{align*}
In turn, because $P(i,v) < 1$ in terms (II), we can use the inequality $1-e^{-x} \geq (1-e^{-1})\,x$, holding when $x \in [0,1]$, with $x = P(i,v)$, thereby concluding that
\[
 q(i,v) \ge (1-e^{-1})P(i,v)
\]
%
Thus
\begin{align*}
    \mathrm{(II)}
\le
    \sum_{v \in V\,:\, P(i,v) < 1} \frac{p(i,v)}{(1-e^{-1})P(i,v)}
\le
    \frac{1}{1-e^{-1}}\,\sum_{v \in V} \frac{p(i,v)}{P(i,v)}
\le
    \frac{\alpha(G)}{1-e^{-1}}~,
\end{align*}
where in the last step we used \cite[Lemma 10]{alon2014nonstochastic}. Notice that despite the statement of this lemma refers to a directed graph and its maximum acyclic subgraph, in the special case of undirected graphs, the size of the maximum acyclic subgraph coincides with the independence number. Moreover, observe that $p(i,1),\dots,p(i,N) \ge 0$ need not sum to one in order for this lemma to hold.
\end{proof}

\proofof{Theorem \ref{th:nontuned}}

\begin{proof}
The standard analysis of the exponentially-weighted algorithm with importance-sampling estimates (see, e.g., the proof of \cite[Lemma~1]{alon2014nonstochastic}) gives for each agent $v$ and each action $k$ the deterministic bound
\begin{equation}
\label{eq:exp3}
    \sum_{t=1}^T \sum_{i=1}^K p_t(i,v) \hloss_t(i,v) \le \sum_{t=1}^T \hloss_t(k,v) + \frac{\ln K}{\eta} + \frac{\eta}{2} \sum_{t=1}^T \sum_{i=1}^K p_t(i,v) \hloss_t(i,v)^2~.
\end{equation}
We take expectations of the three (double) sums in~(\ref{eq:exp3}) separately.
As for the first sum, notice that an iterative application of Lemma~\ref{l:sandwich} gives, for $t > d$,
\[
p_t(i,v) \ge p_{t-d}(i,v) - \eta\,\sum_{s=1}^d p_{t-s}(i,v)\hloss_{t-s}(i,v)~,
\]
so that, setting for brevity $A_t(i,v) = \sum_{s=1}^d p_{t-s}(i,v)\hloss_{t-s}(i,v)$, we have
\begin{align*}
\sum_{t=1}^T \sum_{i=1}^K p_t(i,v) \hloss_t(i,v)
&\geq
\sum_{t=2d+1}^T \sum_{i=1}^K p_t(i,v) \hloss_t(i,v) \\
&\geq
\sum_{t=2d+1}^T \sum_{i=1}^K p_{t-d}(i,v) \hloss_t(i,v)
         - \eta\,\sum_{t=2d+1}^T \sum_{i=1}^K  A_t(i,v)\,\hloss_t(i,v)~.
\end{align*}
Hence
\begin{align*}
\E\left[\sum_{t=1}^T \sum_{i=1}^K p_t(i,v) \hloss_t(i,v) \right]
&\geq
\E\left[ \sum_{t=2d+1}^T \sum_{i=1}^K p_{t-d}(i,v) \hloss_t(i,v)  \right]
-
\eta\,\E\left[\sum_{t=2d+1}^T \sum_{i=1}^K  A_t(i,v)\,\hloss_t(i,v) \right]\\
&=
\E\left[ \sum_{t=2d+1}^T \sum_{i=1}^K p_{t-d}(i,v)\,\E_{t-d}\left[\hloss_t(i,v)\right] \right]\\
&\qquad\qquad
- \eta\,\E\left[\sum_{t=2d+1}^T \sum_{i=1}^K  A_t(i,v)\,\E_{t-d}\left[\hloss_t(i,v)\right] \right]\\
&{\mbox{(since $p_t(i,v)$ is determined by $I_1(\cdot), \ldots, I_{t-d-1}(\cdot)$)}}\\
&=
\E\left[ \sum_{t=2d+1}^T \sum_{i=1}^K p_{t-d}(i,v)\,\loss_{t-d}(i) \right]
-
\eta\,\E\left[\sum_{t=2d+1}^T \sum_{i=1}^K  A_t(i,v)\,\loss_{t-d}(i)\right]\\
&{\mbox{(using (\ref{eq:avevar}))}}\\
&\geq
\E\left[ \sum_{t=1}^T \sum_{i=1}^K p_{t}(i,v)\,\loss_{t}(i) \right]
- 2d
- \eta\,T\,d~.
\end{align*}
The last step uses
\begin{align*}
\E\left[\sum_{i=1}^K A_t(i,v)\,\ell_{t-d}(i) \right]
&\leq \E\left[\sum_{i=1}^K A_t(i,v) \right] \\
&=\E\left[\sum_{i=1}^K \sum_{s=1}^d p_{t-s}(i,v)\hloss_{t-s}(i,v)\right]\\
&= \E\left[\sum_{i=1}^K \sum_{s=1}^d p_{t-s}(i,v)\loss_{t-s-d}(i)\right]\\
&\leq \E\left[\sum_{i=1}^K\sum_{s=1}^d p_{t-s}(i,v)\right]\\
&= d
\end{align*}
holding for $t \geq 2d+1$.
%
%
Similarly, for the second sum in (\ref{eq:exp3}), we have
\begin{align*}
    \E\left[\sum_{t=1}^T \hloss_t(k,v)\right]
=
    \sum_{t=d+1}^T \loss_{t-d}(k)
\le
    \sum_{t=1}^T \loss_t(k)~.
\end{align*}
Finally, for the third sum in (\ref{eq:exp3}), an iterative application of Lemma \ref{l:mult} yields, for $t > d$,
\[
p_t(i,v) \leq \left(1+\frac{1}{d}\right)^d p_{t-d}(i,v) \leq e\,p_{t-d}(i,v)~,
\]
so that we can write
\begin{align*}
    \E\left[\sum_{t=1}^T \sum_{i=1}^K p_t(i,v) \hloss_t(i,v)^2 \right]
&=
    \E\left[\sum_{t=d+1}^T \sum_{i=1}^K \E_{t-d}\left[p_t(i,v) \hloss_t(i,v)^2\right] \right]
\\ &\le
    \E\left[\sum_{t=d+1}^T \sum_{i=1}^K \frac{p_t(i,v)}{q_{d,t-d}(i,v)} \right]
  \qquad\,\,\, \text{(using~(\ref{eq:aveprob}) and $\loss_t(\cdot) \leq 1$)}
\\ &\le
    e\,\E\left[\sum_{t=d+1}^T \sum_{i=1}^K \frac{p_{t-d}(i,v)}{q_{d,t-d}(i,v)} \right],
\end{align*}
the last inequality being due to an iterative application of Lemma~\ref{l:mult}, and the observation that $\left(1+\frac{1}{d}\right)^d \leq e$.

Hence, summing over all agents $v$, dividing by $N$, and using Lemma~\ref{l:q-bound} on $G_{\leq d}$ gives
\begin{align*}
   \frac{1}{N}\,\E\left[\sum_{t=1}^T \sum_{i=1}^K \sum_{v \in V} p_t(i,v) \hloss_t(i,v)^2 \right]
&\le
    \frac{e}{N}\,\E\left[\sum_{t=d+1}^T \sum_{i=1}^K \sum_{v \in V} \frac{p_{t-d}(i,v)}{q_{d,t-d}(i,v)} \right]
\\ &\le
    \frac{e}{(1-e^{-1})\,N}\,\E\left[\sum_{t=d+1}^T\sum_{i=1}^K \left( \alpha(G_{\le d}) + \sum_{v \in V} p_{t-d}(i,v) \right) \right]
\\ &\le
    \frac{e}{1-e^{-1}}\,T\left(\frac{K}{N}\,\alpha(G_{\leq d}) + 1\right)~.
\end{align*}
Finally, putting together as in (\ref{eq:exp3}), setting $\eta = \gamma\big/\bigl(Ke(d+1)\bigr)$, and overapproximating, we obtain the desired bound.
\end{proof}

\proofof{Theorem \ref{th:main}}

\begin{proof}
We start off from first part of the proof of Theorem~\ref{th:nontuned} which, after rearranging terms, gives the following bound for each agent $v$:
\begin{align}
\nonumber
    \E&\left[\sum_{t=1}^T \sum_{i=1}^K p_t(i,v) \loss_t(i)\right] - \sum_{t=1}^T \loss_t(k)
\\ &\le
\nonumber
    2d + \E\left[  \frac{\ln K}{\eta(v)} + \eta(v)\,d^2 + \eta(v)\,\sum_{t=d+1}^T \left(d + \frac{e}{2}\,\sum_{i=1}^K \frac{p_{t-d}(i,v)}{q_{d,t-d}(i,v)}\right) \right]
\\ &\le
\label{eq:Q-bound}
    3d + \E\left[\frac{Ke(d+1)\ln K}{\gamma(v)} + \frac{\gamma(v)}{Ke(d+1)}\sum_{t=1}^T \underbrace{\left(\Ind{t > d}\,d  + \frac{e}{2}\,\left(\sum_{i=1}^K \frac{p_{t-d}(i,v)}{q_{d,t-d}(i,v)}\right)\Ind{t > d}\right)}_{Q_t(v)} \right]~.
\end{align}
Note that the optimal tuning of $\gamma(v)$ depends on the random quantity
\[
    \overline{Q}_T(v) = \sum_{t=1}^T Q_t(v)~.
\]
We now apply the doubling trick to each instance of {\sc Exp3-Coop}. Recall that, for each $v\in V$, we let
$\gamma_r(v) = Ke(d+1)\sqrt{(\ln K)/2^r}$
for each $r = r_0,r_0+1,\dots$, where $r_0 = \bigl\lceil\log_2\ln K + 2\log_2(Ke(d+1))\bigr\rceil$ is chosen in a way that $\gamma_r(v) \le 1$ for all $r \ge r_0$. Let $T_r$ be the random set of consecutive time steps where the same $\gamma_r(v)$ was used. Whenever the algorithm is running with $\gamma_r(v)$ and detects $\sum_{s \in T_r} Q_s(v) > 2^r$, then we restart the algorithm with $\gamma(v) = \gamma_{r+1}(v)$. The largest $r = r(v)$ we need is $\bigl\lceil\log_2\overline{Q}_T(v)\bigr\rceil$ and
\[
    \sum_{r=r_0}^{\bigl\lceil \log_2\overline{Q}_T(v) \bigr\rceil} 2^{r/2} < 5\sqrt{\overline{Q}_T(v)}~.
\]
Because of~(\ref{eq:Q-bound}), the regret agent $v$ suffers when using $\gamma_r(v)$ within $T_r$ is at most
$
    3d + 2\sqrt{(\ln K)2^r}
$.
Now, since we pay at most regret $d$ at each restart, we have
\begin{align*}
    \E\left[\sum_{t=1}^T \sum_i p_t(i,v) \loss_t(i)\right] - \sum_{t=1}^T \loss_t(k)
&\le
    3d + 4Ke(d+1)\ln K
\\ &\quad
    + \E\left[10\sqrt{(\ln K)\overline{Q}_T(v)} + 3d\Bigl\lceil\log_2\overline{Q}_T(v)\Bigr\rceil \right]~.
\end{align*}
The term $3d + 4Ke(d+1)\ln K$ bounds the regret when the algorithm is never restarted implying that only $\gamma_{r_0}(v)$ is used.

Taking averages with respect to $v$, using Jensen's inequality multiple times, and applying the (deterministic) bound
\[
\frac{1}{N}\,\sum_{v\in V} \overline{Q}_T(v)
\leq
\left(d + \frac{e}{2(1-e^{-1})}\,\frac{K\,(\alpha(G_{\le d})+1)}{N} \right)\,T
\]
derived with the aid of Lemma~\ref{l:q-bound} at the end of the proof of Theorem~\ref{th:nontuned}, gives
\begin{align*}
    R_T^{\mathrm{coop}}
&\le
    3d + 4Ke(d+1)\ln K
\\ &\quad
    + 10\sqrt{(\ln K)\E\left[\frac{1}{N}\sum_{v\in V} \overline{Q}_T(v)\right]} + 3d\log_2\left(\E\left[\frac{1}{N}\sum_{v\in V} \overline{Q}_T(v)\right]\right)
\\ &\le
    10\sqrt{(\ln K)\left(d + \frac{e}{2(1-e^{-1})}\,\frac{K\,(\alpha(G_{\le d})+1)}{N}\right)T}
    + 3d\log_2 T + C\,,
\end{align*}
where $C$ is independent of $T$ and depends polynomially on the other parameters.
Hence, as $T$ grows large,
\[
 R_T^{\mathrm{coop}} = \mathcal{O}\left(\sqrt{(\ln K)\left(d+1 + \frac{K}{N}\,\alpha(G_{\le d})\right)T} + d\,\log T \right)~,
\]
as claimed.
\end{proof}

\section{Proofs from Section~\ref{s:many-d}}\label{a:many-d}

We first need to adapt the preliminary Lemmas~\ref{l:sandwich} and~\ref{l:mult} to the new update rule of {\sc Exp3-Coop2} contained in Figure~\ref{f:exp3-coop2}.
\begin{lemma}\label{la:sandwich}
Under the update rule contained in Figure~\ref{f:exp3-coop2}, for all $t \geq 1$, for all $i \in A$, and for all $v \in V$
\begin{align*}
-p_t(i,v)\left(\eta\hloss_t(i,v)+\delta\right) &\leq p_{t+1}(i,v) - p_t(i,v)\\
& \leq p_{t+1}(i,v)\,\sum_{j=1}^{K} p_t(j,v)\left(1 - \Ind{\tprob_{t+1}(i,v) > \delta / K}\big(1-\eta\,\hloss_t(i,v)\big)\right)
\end{align*}
holds deterministically with respect to the agents' randomization.
\end{lemma}
\begin{proof}
For the lower bound, we have
\[
p_{t+1}(i,v) - p_t(i,v) = \frac{\tprob_{t+1}(i,v)}{\TProb_{t+1}(v)} - p_t(i,v) \geq \frac{w_{t+1}(i,v)}{W_{t+1}(v)\,\TProb_{t+1}(v)} - p_t(i,v)~.
\]
Since $W_{t+1}(v) = \sum_{i\in A} p_t(i,v) e^{-\eta \hloss_t(i,v)} \leq \sum_{i\in A} p_t(i,v) = 1$, and $\TProb_{t+1}(v) \leq 1 + \delta$ by~(\ref{e:ptildebound}), we can write
\begin{eqnarray*}
p_{t+1}(i,v) - p_t(i,v)
&\geq&
 \frac{w_{t+1}(i,v)}{1 + \delta} - p_t(i,v)\\
&=&
 p_t(i,v)\left(\frac{e^{-\eta\hloss_t(i,v)}}{1 + \delta} - 1\right)\\
&\geq&
 p_t(i,v)\left(\frac{1 - \eta\hloss_t(i,v)}{1 + \delta} - 1\right)\qquad {\mbox{(using $e^{-x} \geq 1 - x$)}}\\
&\geq&
 p_t(i,v)\left(-\delta -\eta\hloss_t(i,v)\right)
\end{eqnarray*}
as claimed. As for the upper bound, we first claim that
\begin{equation}\label{e:claim}
\frac{w_{t+1}(i,v)}{W_{t+1}(v)} \geq p_{t+1}(i,v)\Ind{\tprob_{t+1}(i,v) > \delta / K}\,.
\end{equation}
To prove (\ref{e:claim}), we recall that $\tprob_{t+1}(i,v) = \max\left\{\frac{w_{t+1}(i,v)}{W_{t+1}(v)},\frac{\delta}{K}\right\}$. Then we distinguish two cases:
\begin{enumerate}
\item If $\frac{w_{t+1}(i,v)}{W_{t+1}(v)} \leq \frac{\delta}{K}$, then
$\tprob_{t+1}(i,v) = \delta / K$, and $w_{t+1}(i,v) / W_{t+1}(v) > 0$ by definition, hence (\ref{e:claim}) holds;
\item If $\frac{w_{t+1}(i,v)}{W_{t+1}(v)} > \frac{\delta}{K}$ then
$\tprob_{t+1}(i,v) = \frac{w_{t+1}(i,v)}{W_{t+1}(v)}$, so that
\(
p_{t+1}(i,v) \leq p_{t+1}(i,v)\,\TProb_{t+1}(v) = \tprob_{t+1}(i,v)
\)
and~(\ref{e:claim}) again holds.
\end{enumerate}
Then, setting for brevity $C = \Ind{\tprob_{t+1}(i,v) > \delta / K}$, we can write
\begin{eqnarray*}
p_{t+1}(i,v) - p_{t}(i,v)
&\leq&
p_{t+1}(i,v) - w_{t+1}(i,v)  \qquad\qquad\qquad\ \text{(from the update~(\ref{eq:exp-upd}))} \\
&\leq&
p_{t+1}(i,v) - W_{t+1}(v) p_{t+1}(i,v)\,C
\qquad\text{(using~(\ref{e:claim}))}\\
&=&
p_{t+1}(i,v)\big(1 - W_{t+1}(v)\,C\big)\\
&=&
p_{t+1}(i,v)\left(\sum_{j\in A}\big(p_t(j,v) - C\,w_{t+1}(j,v)\big)\right)\\
&=&
p_{t+1}(i,v)\,\sum_{j\in A} p_t(j,v)\left(1 - C\,e^{-\eta\,\hloss_t(j,v)}\right)\\
&\le&
p_{t+1}(i,v)\,\sum_{j\in A} p_t(j,v)\left(1 - C(1-\eta\,\hloss_t(j,v))\right)
\end{eqnarray*}
where in the last step we again used $e^{-x} \geq 1 - x$. This concludes the proof.
\end{proof}

\begin{lemma}\label{la:mult}
Under the update rule contained in Figure~\ref{f:exp3-coop2}, if $\delta \leq 1/d(v)$ and $\eta \leq \frac{1}{Ke(d(v)+1)}$, then
\begin{equation}\label{eq:mult-lemma-manyd}
p_{t+1}(i,v) \leq \left(1 + \frac{1}{d(v)}\right)\,p_{t}(i,v)
\end{equation}
holds for all $t \geq 1$ and $i \in A$, deterministically with respect to the agents' randomization.
\end{lemma}
\begin{proof}
If $\tprob_{t+1}(i,v) = \delta / K$ then, from (\ref{e:ptildebound}), we have $\delta / K = p_{t+1}(i,v)\TProb_{t+1}(v) \geq p_{t+1}(i,v)$, and $p_t(i,v) \geq \frac{\delta}{K(1 + \delta)}$.
Hence, $\frac{p_{t+1}(i,v)}{p_{t}(i,v)} \leq \frac{\delta / K}{\delta / (K (1 + \delta))} = 1 + \delta$, so the claim follows from $\delta \leq \frac{1}{d(v)}$.
On the other hand, if $\tprob_{t+1}(i,v) > \delta / K$, then the proof is exactly the same as the proof of Lemma~\ref{l:mult}, for the second inequality in the statement of Lemma~\ref{la:sandwich} turns out to be exactly the same as the corresponding inequality in the statement in Lemma~\ref{l:sandwich}.
\end{proof}

Next, we generalize Lemma~\ref{l:q-bound} to the case of directed graphs. This is where we need a lower bound on the probabilities $p_t(i,v)$. If $G = (V,E)$ is a directed graph, then for each $v \in V$ let $N^-_{\le 1}(v)$ be the in-neighborhood of node $v$ (i.e., the set of $v'\in V$ such that arc $(v',v) \in E$), including $v$ itself.
\begin{lemma}\label{la:q-bound}
Let $G = (V,E)$ be a directed graph with independence number $\alpha(G)$. Let $\bp(v) = \bigl(p(1,v),\dots,p(K,v)\bigr)$ be a probability distribution over $A = \{1,\dots,K\}$ such that $p(i,v) \geq \frac{\delta}{K(1 + \delta)}$. Then, for all $i \in A$,
\[
    \sum_{v \in V} \frac{p(i,v)}{q(i,v)}
\le
    \frac{1}{1-e^{-1}}\left(6\,\alpha(G) \ln \left(1+ \frac{N^2 K(1+\delta)}{\delta}\right)  + \sum_{v\in V} p(i,v) \right)~,
\]
where $q(i,v) = 1 - \prod_{v' \in N^-_{\le 1}(v)}\bigl(1-p(i,v')\bigr)$.
\end{lemma}
\begin{proof}
We follow the notation and the proof of Lemma \ref{l:q-bound}, where it is shown that
\[
    \sum_{v \in V} \frac{p(i,v)}{q(i,v)}
\le
    \frac{1}{1-e^{-1}}\,\sum_{v\in V}\left(\frac{p(i,v)}{P(i,v)} + p(i,v) \right)~.
\]
In order to bound from above the sum $\sum_{v\in V}\frac{p(i,v)}{P(i,v)}$, we combine \citep[Lemma~14 and~16]{alon2014nonstochastic} and derive the upper bound
\[
\sum_{v\in V}\frac{p(i,v)}{P(i,v)} \leq 6\,\alpha(G) \ln \left(1+ \frac{N^2 K(1+\delta)}{\delta}\right)
\]
holding when $p(i,v) \geq \frac{\delta}{K(1 + \delta)}$. Again, the probabilities $p(i,1),\dots,p(i,N) \ge 0$ need not sum to one in order for this lemma to apply.
\end{proof}

With the above three lemmas handy, we are ready to prove Theorem~\ref{t:main-individual}.

\begin{proof}[Theorem \ref{t:main-individual}]
This proof is similar to the proof of Theorem~\ref{th:nontuned}, hence we only emphasize the differences between the two.

From the update rule in Figure~\ref{f:exp3-coop2}, we have, for each $v \in V$,
\begin{eqnarray*}
W_{T+1}(v)
&=&
 {\dt \sum_{i=1}^{K} \frac{\tprob_T(i)}{\TProb_T(v)} e^{-\eta\hloss_T(i,v)}}\\
&\geq&
 \sum_{i=1}^{K} \frac{w_T(i,v)}{W_T(v) \TProb_T(v)}  e^{-\eta\hloss_T(i,v)}
\qquad\qquad\qquad{\mbox{(since $\tprob_T(i) \geq w_T(i,v)/W_T(v)$)}}\\
&=&
  \sum_{i=1}^{K} \frac{\tprob_{T-1}(i,v) e^{-\eta\hloss_{T-1}(i,v)} e^{-\eta\hloss_T(i,v)}}{W_T(v)\TProb_{T-1}(v)\TProb_T(v)}\\
&\vdots&\\
&\geq&
 \sum_{i=1}^{K} \frac{\dt w_1(i,v)\,e^{-\eta \sum_{t=1}^T \hloss_t (i,v)}}{W_1(v) \cdots W_T(v) \TProb_1(v) \cdots \TProb_T(v)}~.
\end{eqnarray*}
Now, because $w_1(i,v) = 1$, $W_1(v) = K$, and $\TProb_t(v) \leq 1+\delta$ for all $t$, see~(\ref{e:ptildebound}), the above chain of inequalities implies that, for any fixed action $k \in A$,
\begin{equation}\label{e:deterministic}
(1 + \delta)^T\, K\,\left(\prod_{t=1}^{T} W_{t+1}(v)\right)
\geq
e^{-\eta \sum_{t=1}^T \hloss_t (k,v)}~.
\end{equation}
%

As usual, the quantity $W_{t+1}(v)$ can be upper bounded as
\begin{eqnarray*}
W_{t+1}(v)
&=&
	\sum_{i=1}^{K} p_t(i,v) e^{-\eta\hloss_t(i,v)}\\
&\leq&
	\sum_{i=1}^{K} p_t(i,v) \left(1 -\eta\hloss_t(i,v) + \frac{\eta^2}{2}\hloss_t(i,v)^2\right)\\
&&{\mbox{(from $e^{-x} \leq 1-x+x^2/2$ for all $x \geq 0$)}}\\
&=&
	1 - \eta\sum_{i=1}^{K} p_t(i,v)\hloss_t(i,v) + \frac{\eta^2}{2}\sum_{i=1}^{K} p_t(i,v)\hloss_t(i,v)^2~.
\end{eqnarray*}
Plugging back into (\ref{e:deterministic}) and taking logs of both sides gives
\[
	T\ln(1+\delta) + \ln K + \sum_{t=1}^T \ln\left( 1 - \eta\sum_{i=1}^{K} p_t(i,v)\hloss_t(i,v) + \frac{\eta^2}{2}\sum_{i=1}^{K} p_t(i,v)\hloss_t(i,v)^2 \right)
\geq
	-\eta \sum_{i=1}^K \hloss_t (k,v)\,.
\]
Finally, using $\ln(1+x) \leq x$, dividing by $\eta$, using $\delta = 1/T$, and rearranging yields
\begin{equation} \label{eq:pre-bound-manyd}
\sum_{t=1}^T \sum_{i=1}^K p_t(i,v) \hloss_t(i,v)
\leq
\frac{1+\ln K}{\eta} + \sum_{t=1}^T \hloss_t (k,v)
+ \frac{\eta}{2} \sum_{t=1}^T \sum_{i=1}^K p_t(i,v) \hloss_t(i,v)^2
\end{equation}
hence arriving at the counterpart to~(\ref{eq:exp3}).

From this point on, we proceed as in the proof of Theorem~\ref{th:nontuned} by taking expectation on the three sums in~(\ref{eq:pre-bound-manyd}). Notice that we do still have, for all $v \in V$, $t > d(v)$, and $i \in A$,
\begin{eqnarray*}
\E_{t-d(v)}\Bigl[\hloss_t(i,v)\Bigr]
&=&
\loss_{t-d(v)}(i)\\
\E_{t-d(v)}\Bigl[p_t(i,v)\hloss_t(i,v)\Bigr]
&=&
p_t(i,v)\loss_{t-d(v)}(i)\\
\E_{t-d(v)}\Bigl[p_t(i,v)\hloss_t(i,v)^2\Bigr]
&=&
p_t(i,v)\frac{\loss_{t-d(v)}(i)^2}{q_{\scP,t-d(v)}(i,v)}~.
\end{eqnarray*}
We can write
\begin{eqnarray*}
\E\left[\sum_{t=1}^T \sum_{i=1}^K p_t(i,v) \hloss_t(i,v) \right]
&\geq&
\E\left[ \sum_{t=1}^T \sum_{i=1}^K p_{t}(i,v)\,\loss_{t}(i) \right] - 2d(v) - (\eta+\delta)\,T\,d(v)\\
\E\left[\sum_{t=1}^T \hloss_t(k,v)\right]
&\le&
\sum_{t=1}^T \loss_t(k)
\end{eqnarray*}
and, as in the proof of Theorem \ref{th:nontuned},
\begin{align*}
    \E\left[\sum_{t=1}^T \sum_{i=1}^K p_t(i,v) \hloss_t(i,v)^2 \right]
\le
    e\,\E\left[\sum_{t=d(v)+1}^T \sum_{i=1}^K \frac{p_{t-d(v)}(i,v)}{q_{\scP,t-d(v)}(i,v)} \right]~.
\end{align*}
Summing over all agents $v$, dividing by $N$, and applying Lemma~\ref{la:q-bound} to the directed graph $G_{\scP}$, the latter inequality gives
\[
\frac{1}{N}\,\E\left[\sum_{t=1}^T \sum_{i=1}^K \sum_{v \in V} p_t(i,v) \hloss_t(i,v)^2 \right]
\le
\frac{e}{1-e^{-1}}\,T\left(\frac{6K}{N}\,\alpha\left(G_{\scP}\right)\,\ln\left(1+2TN^2 K \right) + 1\right)~.
\]
Combining as in (\ref{eq:pre-bound-manyd}), recalling that $\delta = 1/T$, and setting for brevity ${\bar d}_V = \frac{1}{N}\,\sum_{v\in V} d(v)$, we have thus obtained that the average welfare regret of {\sc Exp3-Coop2} satisfies
\begin{align*}
R_T^{\mathrm{coop}}
&\leq 3{\bar d}_V + \eta\,T\,{\bar d}_V + \frac{1+\ln K}{\eta} + \frac{e\eta}{2(1-e^{-1})}\,T\left(\frac{6K}{N}\,\alpha\left(G_{\scP}\right)\,\ln\left(1+2TN^2 K \right) + 1\right)\\
& = \mathcal{O}\left(\eta\,T\,{\bar d}_V + \frac{\ln K}{\eta} + \frac{\eta\,T K}{N}\,\alpha\left(G_{\scP}\right)\,\ln\left(T N K \right)\right)
\end{align*}
as $T$ grows large. This concludes the proof.
\end{proof}

\section{Proofs regarding Section \ref{s:delayed}}\label{a:delayed}

\proofof{Corollary \ref{c:delayed}}

\begin{proof}
In order to prove the upper bound, we use the exponentially-weighted algorithm with Estimate~(\ref{eq:estimator}) specialized to the case of one agent only, namely $B_{t-d}(i) = \Ind{I_{t-d} = i}$ and $q_{d,t-d}(i) = p_{t-d}(i)$. Notice that this amounts to running the standard Exp3 algorithm performing an update as soon a new loss becomes available. In this case, because $N = \alpha(G_{\leq d}) = 1$, the bound of Theorem~\ref{th:nontuned}, with a suitable choice of $\gamma$ (which depends on $T$, $K$, and $d$) reduces to
\[
    R_T = \mathcal{O}\left(d + \sqrt{(K+d)\,T\ln K}\right)~.
\]
We now prove a lower bound matching our upper bound up to logarithmic factors. The proof hinges on combining the known lower bound $\Omega\bigl(\sqrt{KT}\bigr)$ for bandits without delay of~\cite{auer2002nonstochastic} with the following argument by~\cite{weinberger2002delayed} that provides a lower bound for the full information case with delay.
The proof of the latter bound is by contradiction: we show that a low-regret full information algorithm for delay $d>0$ can be used to design a low-regret full information algorithm for the $d=0$ (no delay) setting. We then apply the known lower bound for the minimax regret in the no-delay setting to derive a lower bound for the setting with delay.

Fix $d > 0$ and let $\mathcal{A}$ be a predictor for the full-information online prediction problem with delay $d$. Let $\bp_t$ be the probability distribution used by $\mathcal{A}$ at time $t$. We now apply algorithm $\mathcal{A}$ to design a new algorithm $\mathcal{A}'$ for a full information online prediction problem with arbitrary loss vectors $\bloss_1',\dots,\bloss_B' \in [0,1]^K$ and no delay. More specifically, we create a sequence $\bloss_1,\dots,\bloss_T \in [0,1]^K$ of loss vectors such that $T = (d+1)B$ and $\bloss_t = \bloss_b'$ where $b = \bigl\lceil t/(d+1) \bigr\rceil$. At each time $b=1,\dots,B$ algorithm $\mathcal{A}'$ uses the distribution
\[
    \bp_b' = \frac{1}{d+1} \sum_{s=1}^{d+1} \bp_{(d+1)(b-1)+s}
\]
where $\bp_t = \bigl(\frac{1}{K},\dots,\frac{1}{K}\bigr)$ for all $t \le 1$. Note that $\bp_b'$ is defined using $\bp_{(d+1)(b-1)+1},\dots,\bp_{(d+1)b}$. These are in turn defined using the same loss vectors $\bloss_1',\dots,\bloss_{b-1}'$ since, by definition, each $\bp_{t+1}$ uses $\bloss_1,\dots,\bloss_{t-d}$, and $\bigl\lceil (t-d)/(d+1) \bigr\rceil = b-1$ for all $t = (d+1)(b-1),\dots,(d+1)b-1$. So $\mathcal{A}'$ is a legitimate full-information online algorithm for the problem $\bloss_1',\dots,\bloss_B'$ with no delay. As a consequence,
\begin{align*}
    \sum_{t=1}^T \sum_{i=1}^K \loss_t(i) p_t(i)
&=
    \sum_{b=1}^B \sum_{s=1}^{d+1} \sum_{i=1}^K \loss_b'(i) p_{(d+1)(b-1)+s}(i)
\\ &=
    (d+1)\sum_{b=1}^B \sum_{i=1}^K \frac{1}{d+1}\sum_{s=1}^{d+1} \loss_b'(i) p_{(d+1)(b-1)+s}(i)
\\ &=
    (d+1)\sum_{b=1}^B \sum_{i=1}^K \loss_b'(i) p'_b(i)~.
\end{align*}
Moreover,
\[
    \min_{k \in A} \sum_{t=1}^T \loss_t(k) = (d+1) \min_{k \in A} \sum_{b=1}^B \loss_b'(k)~.
\]
Since we know that for any predictor $\mathcal{A}'$ there exists a loss sequence $\bloss_1',\bloss_2',\dots$ such that the regret of $\mathcal{A}'$ is at least $\bigl(1 - o(1)\bigr)\sqrt{(T/2)\ln K}$, where $o(1) \to 0$ for $K,B \to \infty$, we have that the regret of $\mathcal{A}$ is at least
\[
    (d+1) R_{T/(d+1)}(\mathcal{A}') = \bigl(1 - o(1)\bigr)(d + 1)\sqrt{\frac{T}{2(d+1)}\ln K} = \bigl(1 - o(1)\bigr)\sqrt{(d + 1)\frac{T}{2}\ln K}~,
\]
where $R_{T/(d+1)}(\mathcal{A}')$ is the regret of $\mathcal{A}'$ over $T/(d+1)$ time steps.
The proof is completed by observing that that the regret of any predictor in the bandit setting with delay $d$ cannot be smaller than the regret of the predictor in the bandit setting with no delay or smaller than the regret of the predictor in the full information setting with delay $d$. Hence, the minimax regret in the bandit setting with delay $d$ must be at least of order
\[
    \max\left\{ \sqrt{KT}, \sqrt{(d + 1)T\ln K} \right\} = \Omega\left(\sqrt{(K + d)\,T}\right)~.
\]
\end{proof}

\end{document}